\pdfoutput=1 %% for autotex in arXiv
\documentclass[twoside]{article}

\usepackage[accepted]{aistats2021}
\usepackage[round]{natbib}

\usepackage[utf8]{inputenc} % allow utf-8 input
\usepackage[T1]{fontenc}    % use 8-bit T1 fonts
\usepackage{hyperref}       % hyperlinks
\usepackage{url}            % simple URL typesetting
\usepackage{booktabs}       % professional-quality tables

\usepackage{amsfonts}       % blackboard math symbols
\usepackage{nicefrac}       % compact symbols for 1/2, etc.
\usepackage{microtype}      % microtypography

\usepackage{graphicx}
\usepackage{amsthm}
\usepackage{amsmath}
\usepackage{amssymb} % for boxtimes
\usepackage{mathtools}
\usepackage{comment}
\usepackage{wrapfig}
\usepackage[capitalize]{cleveref}

\renewcommand{\ref}{\cref}

\theoremstyle{plain}
\newtheorem{thm}{Theorem}[section]
\newtheorem{prop}[thm]{Proposition}
\newtheorem{lemma}[thm]{Lemma}
\newtheorem{cor}[thm]{Corollary}

\theoremstyle{definition}
\newtheorem{defn}[thm]{Definition}

\newtheorem{example}[thm]{Example}

\newtheorem{assumption}[thm]{Assumption}
\usepackage{color}
\newcommand{\com}[1]{#1}

\newcommand{\R}{\mathbb{R}}
\newcommand{\N}{\mathbb{N}}
\newcommand{\E}{\mathbb{E}}
\newcommand{\V}{\mathbb{V}}

\newcommand{\mr}[1]{\mathrm{#1} }
\newcommand{\mc}[1]{\mathcal{#1}}

\newcommand{\eps}{\varepsilon}

\DeclareMathOperator{\tr}{\mathrm{tr}}

\newcommand{\pp}[2]{\frac{\partial #1}{\partial #2}}

\newcommand{\pre}{h}
\newcommand{\post}{x}
\newcommand{\act}{\varphi}

\newcommand{\Hmltn}{H}

\newcommand{\affine}[2]{ ( #2 +  #1 \,\boldsymbol{\cdot}\, )_*} %%% author-defined newcommands

\begin{document}



\twocolumn[

\aistatstitle{The Spectrum of Fisher Information of Deep Networks Achieving Dynamical Isometry}

\aistatsauthor{ Tomohiro Hayase \And Ryo Karakida}

\aistatsaddress{ Fujitsu Laboratories \And  AIST} ]

\begin{abstract}
The Fisher information matrix (FIM) is fundamental to understanding the trainability of deep neural nets (DNN), since it describes the parameter space's local metric. 
We investigate the spectral distribution of the conditional FIM, which is the FIM given a single sample, by focusing on fully-connected networks achieving dynamical isometry.
Then, while dynamical isometry is known to keep specific backpropagated signals independent of the depth, we find that the parameter space's local metric linearly depends on the depth even under the dynamical isometry.
More precisely, we reveal that the conditional FIM's spectrum concentrates around the maximum and the value grows linearly as the depth increases.
To examine the spectrum, considering random initialization and the wide limit, we construct an algebraic methodology  based on the free probability theory.
As a byproduct, we provide an analysis of the solvable spectral distribution in two-hidden-layer cases. 
Lastly, experimental results verify that  the appropriate learning rate for the online training of DNNs is in inverse proportional to depth, which is determined by the conditional FIM's spectrum.
%Further, we empirically confirm that the FIM spectrum for multiple samples of data has the same property as that of the single sample version. 
\end{abstract}

\section{Introduction}
Deep neural networks (DNNs) have empirically succeeded in achieving high performances in various machine-learning tasks \citep{LeCun2015DeepLearning, Goodfellow2016DeepLearning}.
Nevertheless, their theoretical understanding has been limited, and their success depends much on the heuristic search setting, such as architectures and hyper-parameters. 
In order to understand and improve the training of DNNs, researchers have developed some theories to investigate, for instance, vanishing/exploding gradient problems \citep{Schoenholz2016DeepPropagation},  the shape of the loss landscape \citep{Pennington2018spectrum, Karakida2019universal}, the global convergence of training and the generalization \citep{Jacot2018NeuralNetworks}.

The Fisher information matrix (FIM) has been a fundamental quantity for such theoretical understandings.  The FIM describes the local metric of the loss surface concerning the KL-divergence function \citep{Amari2016}. 
In particular, the eigenvalue spectrum describes the efficiency of optimization methods. 
For instance, the maximum eigenvalue determines an appropriate size of the learning rate of the first-order gradient method for convergence \citep{Cun1991eigenvalues, Karakida2019universal, wu2019how}. 
In spite of its importance, the spectrum of FIMs in neural networks is not revealed enough from a theoretical perspective. 
The reason is that it has been limited to random matrix theory for shallow networks \citep{Pennington2018spectrum} or mean-field theory for eigenvalue bounds, which may be loose in general \citep{Karakida2019normalization}. 
Thus, we need an alternative approach applicable to DNNs.

It is well known that it is difficult to reduce the training error in very deep models without careful prevention of the vanishing/exploding of the gradient. 
Naive settings (i.e., activation function and initialization) cause  vanishing/exploding gradients, as long as the network is relatively deep. 
The dynamical isometry \citep{Saxe2014exact,Pennington2018emergence} was proposed to solve this problem. The dynamic isometry can facilitate training by setting the input-output Jacobian's singular values to be one, where the input-output Jacobian is the Jacobian matrix of the DNN at a given input.  
Experiments have shown that with initial values and models satisfying dynamical isometry, very deep models can be trained without gradient vanishing/exploding;
\citep{Pennington2018emergence,xiao2018dynamical,sokol2018information} have found that DNNs achieve approximately dynamical isometry over random orthogonal weights, but they do not do so over random Gaussian weights.

%%% conditional
We investigate the asymptotic spectrum of the FIM of multi-layer perceptrons satisfying dynamical isometry in the present work. 
In order to handle the mathematical difficulty to treat the spectrum of the full FIM, we focus on the \emph{conditional FIM} given a single sample. 

%%% fpt
While analyzing the spectrum of DNN, we often face mathematical difficulties caused by the non-linearity of the activation function, the depth of the network, and random matrices.
To handle such difficulty,  we use the free probability theory (FPT). 
The FPT, which is invented by Voiculescu for understanding von Neuman algebras \citep{Voiculescu1985symmetries},  provides algebraic tools of random matrix theory \citep{voiculescu1991limit}.
%In particular, FPT provides tools to compute the asymptotic eigenvalue of random matrix polynomials from that of each component matrix constituting the polynomial.
Since DNN's FIM is a random matrix polynomial of the Jacobian, the FPT provides tools for understanding FIM's spectral distribution.  We use these tools to obtain the propagation of spectral distributions through the layers and then obtain the critical recursive equation of the spectral distributions.

%%% contribution
Our findings are the following. 

The main finding is that the dynamical isometry makes the spectrum of the conditional FIM concentrates on the maximum spectrum, which value grows linearly as the depth increases (\cref{thm:maximum}). 
An interesting phenomenon is that the FIM spectrum depends on the depth, unlike the spectrum of the input-output Jacobian, which is independent of the depth. 
It follows from this that as the depth increases, the parameter space's local geometry linearly depends on the depth and the first-order optimization also suffers from it. 
Now, our approach for the spectral analysis of DNNs consists of three steps. 
In the first step, we consider the dual of the conditional FIM in order to focus on the non-trivial non-zero eigenvalues. 
In the second step, the dual's eigenvalue distribution is decomposed to a free multiplicative convolution of two distributions. Subsequently, we introduce the recursive equations \eqref{align:reccurence_spec} of the dual's spectral distribution throughout the layers. 
In the last step, the induction on depth shows simultaneously that the maximal value of the limit spectrum of the dual is an atom with a large weight, and we get the recursive equation of the maximum \eqref{align:recursive_max}. 
We emphasize that in our setting, the atom helps us simplify the analysis of the maximum. 
n
Secondly,  we discover a solvable case on the spectrum distribution of the dual conditional FIM of a non-trivial DNN. 
In more detail, we explicitly show the asymptotic spectrum of the dual conditional FIM of a two-hidden-layer (a total of three layers) DNN (\cref{thm:two-hidden-layer}). 
As long as we know, this is the first solvable case ever for the FIM's spectrum of the deep architecture and clarifies the connection to the universal law of random matrices in FPT.  

Thirdly, we empirically confirm in \cref{sssec:expected_vs_conditional}  that the spectrum of the FIM at a small number of samples has the same property as the single-sample version. 
\cref{sec:block-diagonal} describes the rationale for a part of the experiment.

Lastly, in \cref{ssec:dynamics}, experimental results confirm that the FIM's dependence on depth determines the appropriate magnitude of the learning rate for the convergence of the first-order optimization at the initial phase of the online training of DNNs.

Our analysis is the first step towards a theoretical understanding of the FIM of DNN achieving dynamical isometry. 
By extending our framework, we expect to see the spectrum of other FIMs in more varied settings.

\subsection{Related Works}
\paragraph{Dynamical Isometry and Edge of Chaos}
\com{ A DNN is said to be on the edge of chaos if it preserves the norm of the gradient and the mean squared singular value of the Jacobian throughout layers.
However, vanishing or exploding of gradients in specific directions still occurs in the worst case. 
To prevent them, we need both orthogonal initialization and weight scale depending on activation functions, and that is the finding of the theory of dynamical isometry \citep{Pennington2017Resurrecting,Pennington2018emergence}.}
The theory of dynamical isometry has been extended to various architectures \citep{burkholz2019initialization, gilboa2019dynamical,tarnowski2019dynamical},  and experiments on CNN showed that the dynamical isometry reduced the training error of 10,000 layers of models \citep{xiao2018dynamical}. 
In contrast, even if the backpropagation signal is isotropic, our study shows that the parameter space's curvature essentially depends on the number of layers. In particular, our study shows that many eigenvalues concentrate on the number of the point of the depth.
For example, it suggests that the learning rate, which has been implicitly set, needs to be set to a smaller and more appropriate value depending on the number of layers.

\paragraph{Fisher Information Matrix of DNN}
Several works have provided the spectral analysis of FIM in limited cases. \citep{Pennington2018spectrum} has analyzed the spectrum of FIM via random matrix theory but limited to shallow networks and random Gaussian weight matrices. 
\citep{Karakida2019normalization, Karakida2019universal} obtained some bounds for the FIM's eigenvalues in DNNs, but their bounds are loose in general and also limited to Gaussian weights. 
\cite{Saxe2014exact} treats the loss's Hessian eigenvalues, but the work is restricted to a linear activation. 
On the contrary, we investigate the FIM spectrum of deep non-linear networks on random orthogonal weights, which satisfy the dynamical isometry. 

\paragraph{Neural Tangent Kernel}
Additionally, let us remark that \citep{Jacot2018NeuralNetworks} uses a version of the dual FIM $\Theta$ with Gaussian initialization as the kernel matrix and call it the neural tangent kernel (NTK). The FIM's eigenvalues also determine convergence characteristics of gradient descent in wide neural networks through the NTK. 
\com{ 
Settings of the last layer of DNNs are different between the theory of NTK and the theory of dynamical isometry.
In the theory of NTK, 
 the dimension of the final layer of the DNN is set to be lower-order than that of hidden layers. However, in the theory of dynamical isometry, the final layer's dimension is set to be of the same order as the hidden layers.
}
%%% last dimension が違う. % 直交初期化は扱われているが、dynamical isometryはやっていない。Lに比例しているのは同様。我々はさらに、dynamical isometryで最大固有値に集中する現象を示している。
\com{ \cite{huang2020neural} treats the NTK regime of orthogonal initialization but does not analyze the case of dynamical isometry.}
\section{Preliminaries}
\subsection{Settings}

\paragraph{Spectral Distribution}
For  $M \times M$ symmetric matrix $A$ with $M \in \N$, its spectral distribution $\mu_A$ is given by 
\begin{align}
    \mu_A = \frac{1}{M} \sum_{k=1}^M \delta_{\lambda_k},
\end{align}
where  $\tr$ is the normalized trace, $\lambda_k (k=1, \dots, M)$ are eigenvalues of $A$, and $\delta_\lambda$ is the discrete probability distribution whose support is $\{\lambda\} \subset \R$,
In general, for a noncommutative probability space $(\mc{A}, \tau)$ (see supplementary material C), 
the  spectral distribution $\mu$ of $A \in \mc{A}$ is a probability distribution $\mu$ on $\R$ satisfying $\tau(A^m) = \int t^m\mu(dt)$ for any $m \in \N$.

\paragraph{Dynamical Isometry}
\com{We say that a feed-forward network achieves dynamical isometry if all singular values of the Jacobian of the network on an input are equal to one.
In a later section, we introduce an approximate dynamical isometry in a similar setting as  \citep{ Pennington2018emergence}.
%%% 強調する場合はperfect dynamical isometryとよぶことにする。
%%%%
% approximate dynamical isometry
%% [icml2014] act as a near isometry, up to some overall global O(1) scaling, on a subspace of as high a dimension as possible. This is equivalent to having as many singular values of the product of Jacobians as possible within a small range around an O(1) constant
%% [emergence] d.i. の定義は書いてない。　stable limit distribution
%% Lに依存しない
%[resurrection of sigmoid]These random weight initializations were primarily driven by the principle that the mean squared singular value of a deep network's Jacobian from input to output should remain close to 1. 
}

\paragraph{Network Architecture}
We assume random weight matrices as is usual in the studies of FIM \citep{Pennington2018spectrum, Karakida2019universal} and dynamical isometry \citep{Saxe2014exact, Pennington2018emergence}. 
Fix $L \in \N$. We consider an $L$-layer feed-forward neural network $f_\theta$ with $M \times M$ weight matrices $W_1, W_2, \dots, W_L$ and pointwise activation functions $\act^1, \dots \act^{L-1}$ on $\R$.
Besides, we assume that $\act^\ell$ is continuous and differentiable except for finite points.
Firstly, pick a single input  $x \in \R^M$.
Set $\post^0=x$. 
For $\ell=1, \dots, L$, set
\begin{equation}
    \pre^\ell  = W_\ell \post^{\ell-1} + b^\ell, \ \ 
      \post^\ell  = \act^\ell(\pre^\ell ).
    \label{align:post-to-pre}
\end{equation}
We omit the bias parameters $b^\ell$ in \eqref{align:post-to-pre} to simplify the analysis. 
Write $f_\theta(x) = h^L$.
Write
\begin{align}
    D_\ell = \pp{\post^\ell}{\pre^\ell},  \   \delta_{L \to \ell}  = \pp{\pre^L}{\pre^\ell}.
\end{align}

\paragraph{Fisher Information Matrix}
We focus on the \emph{the Fisher information matrix} (FIM) for supervised learning with a mean squared error (MSE) loss \citep{Pennington2018spectrum, Karakida2019normalization, Pascanu2013revisiting}. 
Let us summarize its definition and basic properties. 
Given $x \in \R^M$ and $\theta$, we consider a Gaussian probability model 
%\begin{align}\label{align:gaussian-model}
$p_\theta(y | x)  = \exp\left(- \mathcal{L}\left(f_\theta(x) -y\right) \right)/\sqrt{2\pi}$  $ (y \in  \R^M)$.
%\end{align}
We define the MSE loss by  
\begin{align}
    \mathcal{L}(u) =||u||^2/2,  (u \in \R^M),
\end{align}
where $||\cdot ||$ is the Euclidean norm.
In addition, consider  a probability distribution $p(x)$ and
 a joint distribution  $p_\theta(x,y) = p_\theta(y|x)p(x)$.
Then, the FIM is defined by
%\begin{align}\label{align:original-fim}
 $   \mathcal{I}(\theta) = \int  [\nabla_\theta \log p_\theta(x,y )^\top \nabla_\theta \log p_\theta(x,y)] p_\theta(x,y)dxdy$, which is an $LM^2 \times LM^2$ matrix.
%\end{align}
Now,  we denote by $\mc{I}(\theta | x)$ the \emph{conditional FIM} (or pointwise FIM) given a single input $x$ defined by 
\begin{align}
    &\mc{I}(\theta | x)  =  \nonumber\\
    &\int [\nabla_\theta \log p_\theta(y | x)^\top \nabla_\theta \log p_\theta(y | x)] p_\theta(y|x)dy.
\end{align}
Since $p_\theta(y|x)$ is Gaussian, the conditional FIM is equal to
\begin{align}\label{align:cfim}
 \mc{I}(\theta | x) = \pp{ f_\theta(x)}{\theta}^\top \pp{f_\theta(x)}{\theta}.
\end{align}
We mainly investigate this conditional FIM in the following analysis. 
Since the distribution $p(x)$ of the input does not depend on $\theta$, the FIM is given by
\begin{align}\label{align:fim}
    \mc{I}(\theta)  = \int \pp{ f_\theta(x)}{\theta}^\top \pp{f_\theta(x)}{\theta}p(x)dx.
\end{align}
We regard $p(x)$ as an empirical distribution of input samples and the FIM \eqref{align:fim} is usually referred to as the empirical FIM \citep{kunstner2019limitations,Pennington2018spectrum,Karakida2019normalization}.
%
\begin{comment}
Now consider the empirical distribution of input samples.
Since the FIM's original definition is defined only for absolutely continuous distributions of $(x,y)$, we extend the definition of the FIM of DNN via \eqref{align:fim} for an arbitrary distribution of the input $x$. 
Then we consider the FIM by replacing $p(x)$ with the empirical distribution of input samples.
\end{comment}
As is known in information geometry \citep{Amari2016}, the FIM works as a degenerate metric on the parameter space: the Kullback-Leibler divergence between the statistical model and itself perturbed by $d\theta$ is given by
$ D_{\mathrm{KL}}(p_\theta || p_{\theta+ d\theta}) = d\theta^\top   \mathcal{I}(\theta) d\theta.$ 
More intuitive understanding is that we can write the Hessian of the loss as 
%\begin{align}
 $\pp{}{\theta}^2\E_{x,y}[\mathcal{L}(f_\theta(x) - y)] \notag     =    \mathcal{I}(\theta) + \E_{x,y}[ (f_\theta(x) -y )^\top     \pp{}{\theta}^2  f_\theta(x) ]$.
%\end{align}
Hence the FIM also characterizes the local geometry of the loss surface around a global minimum with a zero training error.

\paragraph{Dual Fisher Information Matrix}
Now, in order to ignore  $\mathcal{I}(\theta|x)$'s  trivial eigenvalue zero,
we introduce the \emph{dual conditional FIM} given by
\begin{align}\label{align:H-simplfied}
    H_L(x,\theta) = \frac{1}{M} \pp{f_\theta(x)}{\theta} \pp{f_\theta(x)}{\theta}^\top,
\end{align}
which is an $M \times M$ matrix.
If there is no confusion, we omit the arguments and denote it by $H_L$. 
Except for trivial zero eigenvalues, $\mathcal{I}(\theta|x)/M$ and $H_L(x,\theta)$ share the same eigenvalues as the following:
\begin{align}\label{align:dual-formula}
    \mu_{I(\theta|x)/M} = \frac{LM^2 -M }{LM^2} \delta_0 + \frac{1}{L} \mu_{H_L(x, \theta)},
\end{align}
where $\mu_A$ is the spectral distribution for a matrix $A$.
Note that we multiplied the normalization factor $1/M$ in \eqref{align:H-simplfied} because the loss $\mc{L}(y)$ is $O(M)$ as $M \to \infty$ when the output $y$ has the constant order second moments.

\subsection{Free Probability Theory}
\paragraph{Freeness}
\com{Free probability theory gives an asymptotic analysis of families of random matrices in the infinite-dimensional limit.
The asymptotic freeness is a vital notion of free probability theory to separate the random matrices' spectral analysis into their respective spectral analysis. We refer readers to the supplemental material C and \citep{voiculescu1992free, Mingo2017free} for more information on the asymptotic freeness.}

\paragraph{S-transform}
Given probability distribution $\nu$, set $G_\nu(z )  = \int (z-t)^{-1} \nu(dt)$ and $h_\nu(z) = z G_\nu(z) -1$. Then the S-transform   \citep{Voiculescu1987multiplication}  of $\nu$ is defined as 
\begin{align}
    S_\nu(z) =  \frac{1+z}{z}\frac{1}{h_\nu^{-1}(z)}.  \label{align:S}
\end{align}
For example, given discrete distribution $\nu = \alpha\delta_0 + (1- \alpha)\delta_\gamma$ with $0\leq\alpha \leq 1$ and $\gamma>0$, we have
%\begin{align}\label{align:S-two-atom}
 $ S_\nu(z) = \gamma^{-1}(z+\alpha)^{-1}(z+1)$.
%\end{align}
If two operators $A$ and $B$ are free and each spectral distribution is given by $\mu$ and $\nu$ respectively, then the spectral distribution of $AB$ is given by the free multiplicative convolution, denoted by $\mu \boxtimes \nu$ \citep{Voiculescu1987multiplication}. Moreover, it holds that
\begin{align}\label{align:S-trans-decompose}
    S_{\mu \boxtimes \nu}(z) =  S_{\mu}(z)S_{\nu}(z).
\end{align}

\section{Propagation  of Spectral Distributions}
\subsection{Recursive Equations}
We use several assumptions in the mean-field theory of neural networks  \citep{Pennington2017Resurrecting, Pennington2018emergence, Karakida2019universal} used in the analysis of dynamical isometry.
Firstly, we assume that  $W_\ell/\sigma_\ell$ are independent and  uniformly distributed on $M \times M$ orthogonal matrices, where $\sigma_1, \dots, \sigma_L > 0$ are constant. 
Secondly, set
\begin{align}
\hat{q}_\ell = ||x_{\ell}||^2/M.     
\end{align}
Assume that $\hat{q}_0$ converges to $q_0 > 0$.
With an appropriate choice of activation function, the empirical distribution of each hidden unit $x_\ell$ converges to the centered normal distribution \citep{Pennington2018emergence}.
Set $q_\ell = \lim_{M\to\infty} \hat{q}_\ell$.
Lastly, we assume the following asymptotic freeness. 
\begin{assumption}\label{assmptn:freeness}
We assume that $( D_\ell)_{\ell=1}^{L-1}$ is asymptotically free from $(W_\ell, W_\ell^\top)_{\ell=1}^L$ as $M\to \infty$ almost surely.
\end{assumption}
Note that  \cref{assmptn:freeness}  is weaker than the assumption of the forward-backward independence that researches of dynamical isometry assumed  \citep{Pennington2018emergence, Pennington2017Resurrecting, Karakida2019universal}.
Several works prove or treat the asymptotic freeness with Gaussian initialization \citep{Hanin2019products, Yang2019scaling, Pastur2020random}, and we expect that it will also hold with orthogonal initialization.

Now we have prepared to discuss the propagation of spectral distributions.
It holds that
\begin{align}
    \Hmltn_L =  \sum_{\ell=1}^L \hat{q}_{\ell-1}\delta_{L \to \ell} \delta_{L \to \ell}^\top.
\end{align}
Since $ \delta_{L \to \ell} = W_L D_{L-1}\delta_{L-1 \to \ell}$ $(\ell < L)$, it holds that
\begin{align}\label{align:reccurence_matrix}
\Hmltn_{\ell +1 } =   \hat{q}_{\ell}I + W_{\ell+1} D_{\ell}\Hmltn_{\ell}D_{\ell}W_{\ell+1}^\top, 
\end{align}
where $I$ is the identity matrix.
Let $\mu_\ell$ (resp.\,$\nu_\ell$) be the limit spectral distribution as $M\to\infty$ of $H_\ell$ (resp.\,$D_\ell^2$). Note that $\mu_1 = \delta_{q_0}$.
By \cref{assmptn:freeness}, we have the following propagation equation of limit spectral distributions.  For $\ell=1, \dots, L-1$, we have
\begin{align}\label{align:reccurence_spec}
    \mu_{\ell+1} =  \affine{\sigma_{\ell+1}^2}{q_\ell}( \nu_\ell \boxtimes \mu_\ell ),
\end{align}
where  the distribution $(b+a\,\cdot\,)_*\mu$ is the pushforward of $\mu$ with the map $x \mapsto b + ax$ for a given distribution $\mu$. %See supplemental material D for more information on the pushforward.

\subsection{An Example: The Two-Hidden-Layer Case}
To show a nontrivial example, we examine the solvable asymptotic spectrum of the conditional FIM in the case of a two-hidden-layer network (i.e.\,$L=3$).   
Assume that \begin{align}\label{align:dist-of-nu}
     \nu_\ell =  (1-\alpha_\ell)\delta_0 + \alpha_\ell \delta_{\gamma_\ell} , 
\end{align}
where $0< \alpha_\ell < 1$ and $\gamma_\ell>0$.
We get the distribution \eqref{align:dist-of-nu} if we choose activation  as the shifted-ReLU ($\act(x) = ax $ if $x > b$ 
otherwise $ab$ with $a,b>0$) or the hard tanh  given by
\begin{align} \label{align:Nhtanh}
\act_{s,g}(x) = \begin{cases} 
g x, & \text{ if } sg|x| < 1,\\
g \cdot \mathrm{sgn}(x), & \text{otherwise},
\end{cases}
\end{align}
where $s,g>0$. These activation functions appear in \citep{Pennington2018emergence} for dynamical isometry.
Then we have the following explicit representation of the $H_3$'s asymptotic spectral distribution $\mu_3$.
\begin{thm}\label{thm:two-hidden-layer}
We have $\mu_3(dx) = \mu_\mr{atoms}(dx) + \rho(x)dx$, where
  \begin{align} 
\mu_\mr{atoms} &= (1- \alpha_2)\delta_{\lambda_{\min}} 
+ (\alpha_2 - \alpha_1)^+ \delta_{\lambda_\mathrm{mid}} \notag \\
&+   (\alpha_1 + \alpha_2 -1)^+ \delta_{ \lambda_{\max} }(dx), \label{thm:two-hidden-layer:main}\\
\rho(x) &=  \frac{\sqrt{  (\lambda_+ -x)(x - \lambda_- ) }}{2\pi (x- \lambda_\mathrm{mid}  )(  \lambda_{\max}  - x)}{\bf 1}_{[\lambda_-, \lambda_+] }(x), 
\end{align}
with the following notations: $a^+ =  \max(a,0)$ for $a \in \R$, ${\bf 1}_X$ is the indicator function for $X \subset \R$, 
\begin{align}
    \lambda_{\min} &= q_2, \label{thm:two-hidden-layer:min}\\
    \lambda_\mathrm{mid} &= q_2 +\sigma_3^2\gamma_2 q_1, \label{thm:two-hidden-layer:mid}\\
    \lambda_\pm   &=  q_2 +   \notag \\
     &\sigma_3^2\gamma_2   \{ q_1 +\notag \\
     &\sigma_2^2\gamma_1 q_0 [ \sqrt{\alpha_1(1 - \alpha_2)} \pm\sqrt{ \alpha_2(1 - \alpha_1)} ]^2 \}, \label{thm:two-hidden-layer:pm} \\
    \lambda_{\max} &= q_2 +   \sigma_3^2\gamma_2( q_1  +   \sigma_2^2\gamma_1 q_0).\label{thm:two-hidden-layer:max}
\end{align}

\end{thm}

\begin{proof}
The proof is based on \eqref{align:S-trans-decompose} and  \eqref{align:reccurence_spec},  and is postponed to the supplemental material A.
\end{proof}

\begin{figure*}[t]
    \centering
    \includegraphics[width=0.32\linewidth]{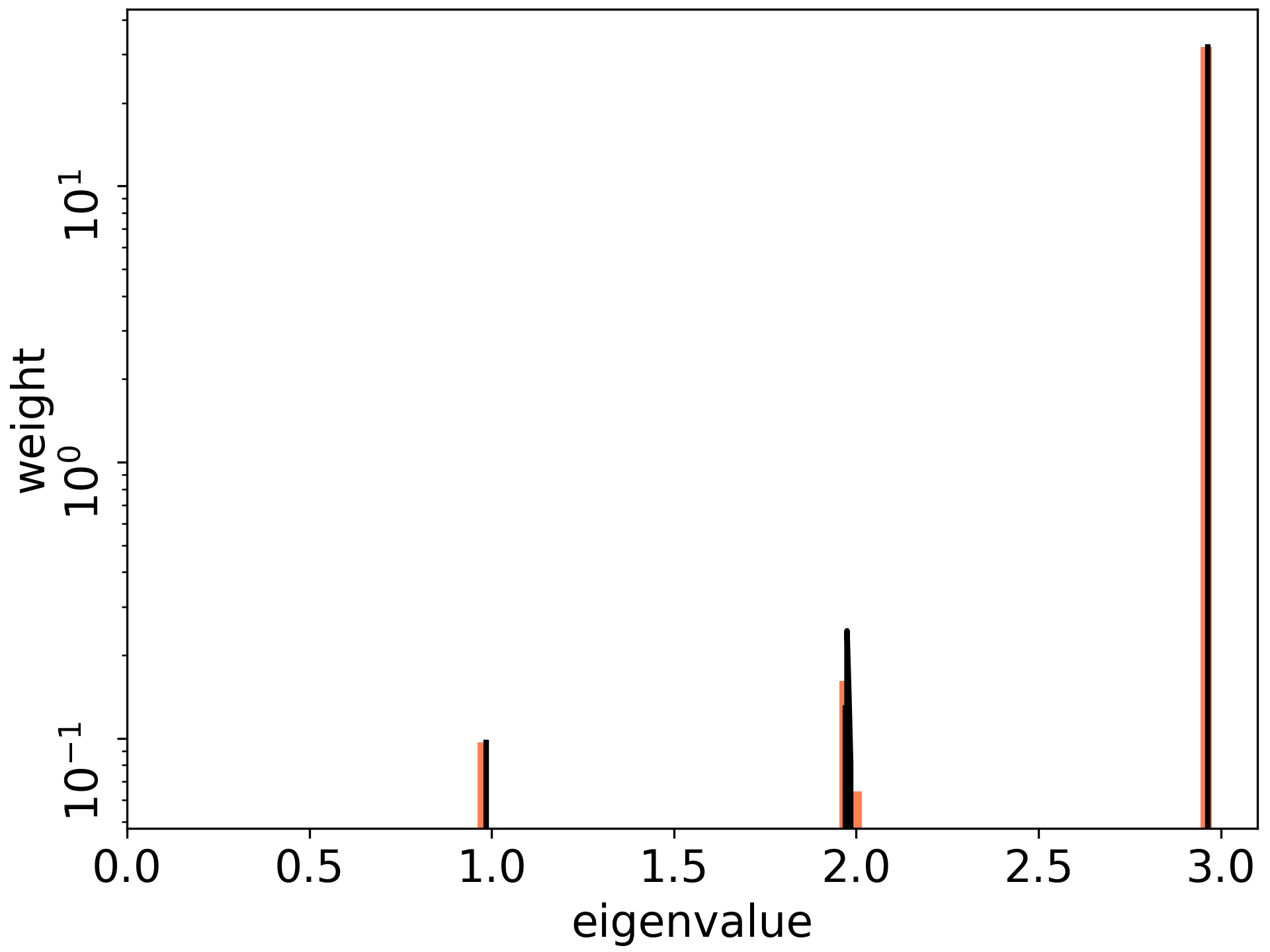}
    \includegraphics[width=0.32\linewidth]{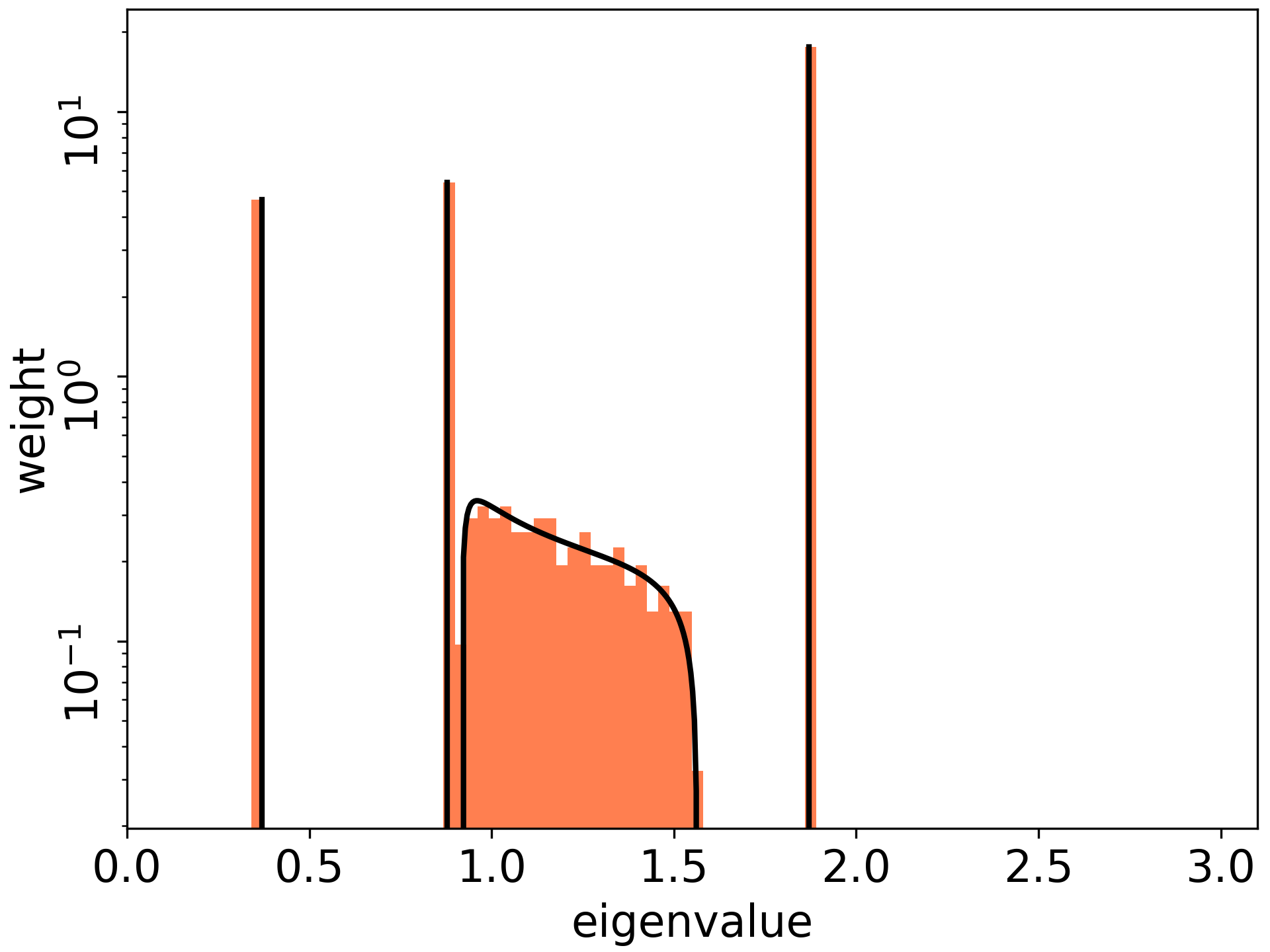}
    \includegraphics[width=0.32\linewidth]{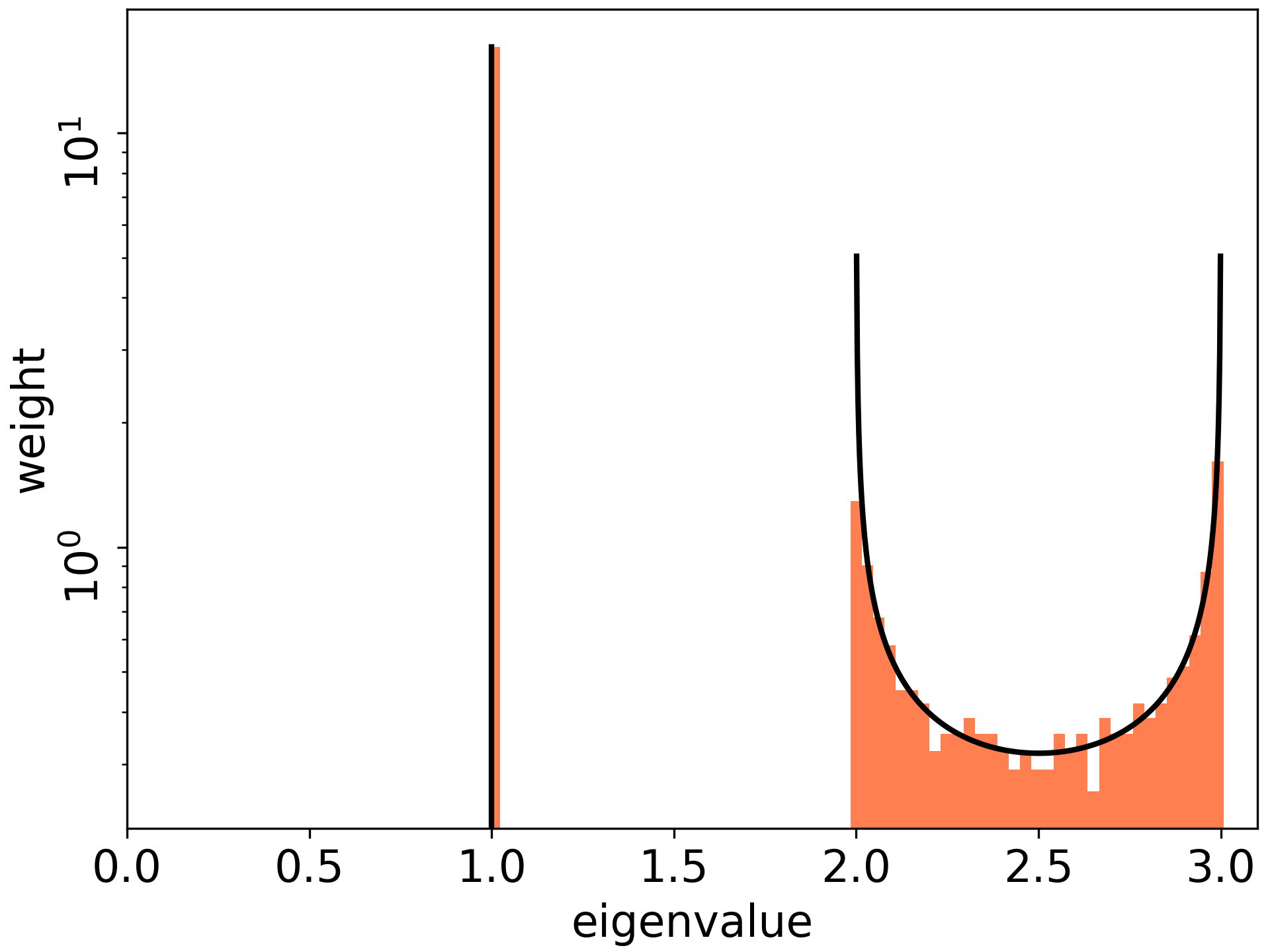}
    \caption{ Normalized histograms of eigenvalues of $H_3$ (orange) and  predicted $\mu_3$ by \cref{thm:two-hidden-layer} (black lines). The y-axis in each figure is logarithmic.
    We set $M=1000$, $\hat{q}_0=\sigma_\ell=1$, the width of bins $0.31$, and used the following setting.
    (Left): We used the hard tanh activation \eqref{align:Nhtanh} with $s=0.125$ and $g=1.0013$ to achieve dynamical isometry.
    (Center): We used the hard tanh with $s=g=1$. 
    (Right):  We constructed $H_L$ based on \eqref{align:reccurence_matrix} by replacing each Jacobian $D_\ell$ with an independent matrix whose spectral distribution is $(1/2)\delta_0 + (1/2)\delta_1 $. 
    }
    \label{fig:two-hidden-layer}
\end{figure*}

\cref{fig:two-hidden-layer} shows the agreement of the predicted distribution $\mu_3$ by \cref{thm:two-hidden-layer} and the empirical spectral distribution of $H_3$ . 

\cref{thm:two-hidden-layer} \eqref{thm:two-hidden-layer:max} reveals that the maximal eigenvalue is close to three, which is the number of layers, when \eqref{align:sigma-gamma-one} is satisifed and $q_*=1$.
Further, if \eqref{align:alpha-one} is satisfied,  then $(\alpha_1 + \alpha_2 -1)^+$ is close to one  and the eigenvalues concentrate on the maximal eigenvalue.
We observed in \cref{fig:two-hidden-layer} (left) that most of the eigenvalues concentrated at the value of depth when a DNN achieved dynamical isometry approximately, but there were other peaks. Later, we show in \cref{thm:maximum} that the eigenvalues concentrate on the value of depth at the large depth. 

Additionally,  \cref{thm:two-hidden-layer} \eqref{thm:two-hidden-layer:main}  reveals that the weight of the minimum eigenvalue depends on how much the last activation's derivation vanishes.

\cref{thm:two-hidden-layer} also gives insight into the spectrum of the DNN out of the dynamical isometry.
Although \cref{fig:two-hidden-layer} (right) is also out of dynamical isometry,  the spectrum obeys the arcsin law known in FPT \citep{voiculescu1992free} and is interesting its own right. 
Once the hyperparameters are standardized and each $D_\ell$ is a projection, the spectrum is attributed to the product of free two projections well examined in FPT.

\section{Analysis through Approximate Dynamical Isometry}
\subsection{Assumptions}\label{ssec:assumptions}
Let us review on how to achieve the dynamical isometry.
For the sake of the prospect of the theory, let $J$ be the Jacobian of the network with ignoring the last layer $\pre^L \mapsto \post^L = W_L \pre^L$.
We say that the network achieves dynamical isometry if all eigenvalues of $JJ^T$ are one.
Consider the limit spectral distribution $\nu_\ell$ of $D_\ell^2$ given by \eqref{align:dist-of-nu}.
At the deep limit, we consider the situation such that forward and backward signal propagation are stable. Hence  we adopt the following assumption,
\begin{assumption}\label{assumption:simplify}\citep{Pennington2018emergence}
For each $L \in \N$,  sequences $q_\ell, \alpha_\ell, \gamma_\ell$ and $\sigma_{\ell+1}$ ($\ell=0,1,\dots,L-1$) are constant for every $\ell$, but depend on $L$.  In addition, each constant converges to a finite value as $L\to\infty$.  %(We use $q_{L-1}, \alpha_{L-1}, \gamma_L$ and $\sigma_{L+1}$ to write the constants.)
\end{assumption}

Now the limit spectral distribution $\mu_{J^TJ}$ of  $J^\top J$ as $M \to \infty$ is given by $ [(\sigma_{L-1}^{2}\,\cdot\,)_* \nu_{L-1}]^{\boxtimes (L-1)}$. By \eqref{align:S-trans-decompose}, 
\begin{align}
S_{\mu_{J^\top J}}(z) = (\frac{1}{ \sigma_{L-1}^{2} \gamma_{L-1}}    (1 +  \frac{1 - \alpha_{L-1}}{z +\alpha_{L-1}}))^{L-1}.
\end{align}
In order to achieve dynamical isometry, we need that $[(\sigma_{L-1}^{2}\,\cdot\,)_* \nu_{L-1}]^{\boxtimes (L-1)}$ converges to a compactly supported distribution as $L \to \infty$, and $S_{JJ^T}(z)$ converges to a non-zero function of $z$.  
Since the right hand side is approximated by $\exp (  L ( 1- \alpha_{L} )(z+\alpha_{L})^{-1}  - L\log \sigma_L^2\gamma_L) )$ as $L \to \infty$, we need
\begin{align}
 \log \sigma_L^2 \gamma_L  &=  O(L^{-1}) \label{align:sigma-gamma-one}\\    
 1 -\alpha_L &= O(L^{-1}) \label{align:alpha-one}
\end{align}
 and  as $L \to \infty$.

Now, the exact dynamical isometry, which means that the input-output Jacobian is an orthogonal matrix, is too strong. Thus we consider the following approximate condition. 
\begin{assumption}\label{assumption:eps}
We assume that the following limits exist and $|\eps_1|, |\eps_2| < 1$:
\begin{align}
\eps_1 &= \lim_{L\to\infty}L(1-\alpha_L),\\
\eps_2 &= -\lim_{L\to\infty}L \log\sigma_{L}^2\gamma_L.
\end{align}
\end{assumption}
To achieve exact dynamical isometry, we need  $\eps_1 = \eps_2 = 0$.  \cref{assumption:eps} is implicitly assumed in \cite{Pennington2018emergence}.

\subsection{Concentration around the Maximal Eigenvalue of the Conditional FIM}
%%% 一般の層数のネットワークについて、上で導入した漸化式を使って極限固有値分布の平均と最大を解析する： (mean, maxの定義)
%%% （最大固有値の極限と極限分布の最大スペクトルは違うということを強調しておく。）
%%% 
We investigate $||\mu_L||_\infty$, where we denote by $||\nu||_\infty$ the maximum of the support of $\nu$ for a compactly supported probability distribution $\nu$ on $\R_+$.
We show that $||\mu_L||_\infty$ is $O(L)$ as $L \to \infty$. 
We discuss the intuitive reason for this ahead of time.
Now, $\nu_\ell$  is close to a delta measure at a point to achieve dynamical isometry. 
Then the recursive equation \eqref{align:reccurence_spec}  looks like an affine transform.
Therefore, $||\mu_L||_\infty$ is the result of $L$-times affine transforms, hence  it is $O(L)$ as $L \to \infty$. 

A key of the proof is to show that $||\mu_L||_\infty$ is an atom of $\mu_L$ under \cref{assumption:eps}, which is an assumption to achieve approximate dynamical isometry.
Recall that $x \in \R$ is an atom of a probability distribution $\nu$ if and only if $\nu(\{x\}) > 0$.
%Note that the phenomenon of the maximum being an atom is non-trivial.
Here we sketch the proof of the phenomenon.
Firstly, the support of $\nu_\ell$, which is the limit spectral distribution of $D_\ell^2$, consists of two atoms $0$ and $\gamma_\ell$, and the weight $\alpha_\ell$ of $\gamma_\ell$  is close to one for sufficiently large $\ell$.
Secondly, we recall the following fact of the free multiplicative convolution: if $a$ (resp.\,$b$) is an atom of a probability  distribution $\mu$ (resp.\,$\nu$) and $\mu(\{a\}) + \nu(\{b\}) - 1 >0$, then $ab$ is an atom of $\mu \boxtimes \nu$ and 
\begin{align}\label{align:free-multi-conv}
    \mu \boxtimes \nu ( \{ab\}) = \mu(\{a\}) + \nu(\{b\}) - 1.
\end{align}
Thus, if $||\mu_\ell||_\infty$ is an atom of $\mu_\ell$ with sufficiently large weight, then $||\nu_\ell||_\infty || \mu_\ell||_\infty$ is an atom of $ \nu_\ell \boxtimes \mu_\ell$.
Note that \eqref{align:free-multi-conv} is different from that in the computation rule of the classical convolution.
Additionally, by the recursive equation \eqref{align:reccurence_spec}, we have the following recursive equation:
\begin{align}\label{align:recursive_max}
    ||\mu_{\ell+1}||_\infty = q_\ell  +  \sigma_{\ell+1}^2  ||\nu_\ell ||_\infty ||\mu_\ell||_\infty.
\end{align}
This recursive equation allows us to use induction.
Then we have our theorem.
\begin{thm}\label{thm:maximum}
Consider \cref{assumption:simplify}  and  \ref{assumption:eps}.
Then for sufficiently large $L$,  it holds that $||\mu_L||_\infty$  is an atom of $\mu_L$ with weight $1- (L-1)(1-\alpha_{L-1})$,  and 
\begin{align}
     \lim_{L\to\infty}L^{-1}||\mu_L||_\infty  = q  \frac{  1 - \exp\left(-\eps_2\right) }{ \eps_2}.
\end{align}
In particular, the limit has the expansion $q(1 - \eps_2/2) + O(\eps_2^2)$ as the further limit $\eps_2 \to 0$.
\end{thm}
The proof is postpone to Appendix D.
%% ここで証明のスケッチを述べる。厳密な証明はAppendixにpostponeする.
%% まず、乗法的に自由畳込みに関する以下の事実がある：
%%% もし a,bが確率分布\mu, \nuのatomで、\mu({a}) + \nu({b}) - 1 > 0ならば、abは自由乗法的畳み込み \mu \boxtimes \nu のatomであり、
%\begin{align}
%    \mu \boxtimes \nu ( \{ab\}) = \mu(\{a\}) + \nu(\{b\}) - 1 .
%\end{align}
%%% classical multiplicative convolution とは、その重みの計算式が異なることに注意する。

%%% この重みの計算式を使って、仮定を使って、次の３つのことを同時(simaltaniously)に$\ell$に関して帰納的に示す。
%%% (1) || ||と書く) は常にatomである（つまり点スペクトル）ことと、
%%% (2) weighは〜を満たす
%%% (3) 漸化式 ~ が成り立つ.

%あとは、(3)の漸化式をsolveして、評価を得る.
Recall that $\mu_L$ is equal to the limit distribution of eigenvalues of the conditional FIM $I(\theta|x)$ except for $LM^2 -M$ zeros.
\cref{thm:maximum} shows that the maximum of the support is $O(L)$ as $L \to \infty$.
Furthermore, we emphasize that the weight $1 - (L-1)(1-\alpha_{L-1}) \sim 1 - \eps_1$ of the maximal eigenvalue $||\mu_{L}||$ is close to $1$.
Therefore, non-zero eigenvalues of the conditional FIM  asymptotically concentrates around $qL(1-\eps_2/2)$. 
In particular,  the non-degenerate-part of the conditional FIM is approximated by a scaled identity operator.

\subsection{Expected vs Conditional FIM}\label{sec:block-diagonal}
While we have investigated the conditional FIM \eqref{align:cfim} so far, it will be curious to show how the obtained eigenvalue statistics could be related to those of the FIM \eqref{align:fim}.
We investigate the analysis of mean eigenvalues of both the conditional FIM and the usual FIM.
Analysis of the mean is easier than analysis of the maximum value, but is done to observe the effect of the dynamic isometry.
Now, we denote by $m_1(\mu)$ the mean of a distribution $\mu$ given by $m_1(\mu) = \int x \mu(dx)$.
We show that the mean of spectrum is $O(L)$ as $L \to \infty$.

Applying the decomposition formula of the free multiplicative convolution $m_1(\mu \boxtimes \nu)=m_1(\mu)m_1(\nu)$ to \cref{align:reccurence_spec}, we have the following recurrence formula of the mean:
\begin{align}
    m_1(\mu_{\ell + 1}) = q_\ell + \sigma_{\ell+1}^2 m_1(\nu_\ell) m_1(\mu_\ell).
\end{align}
Then we have the followings.
\begin{prop}\label{prop:expectation}
Under  \cref{assumption:simplify} and \ref{assumption:eps}, it holds that
\begin{align}
\lim_{L \to \infty}L^{-1} m_1(\mu_L)  =  q \frac{1-  \exp(-\eps_1 - \eps_2)}{\eps_1 + \eps_2},
\end{align}
where $q = \lim_{L\to\infty}q_L$. In particular, the limit has the expansion $ q (1 -  (\eps_1 + \eps_2)/2) + O( (\eps_1 + \eps_2)^2 )$ as  $\eps_1, \eps_2 \to 0$.
\end{prop}
\begin{proof}
We have $m_1(\mu_{L+1}) =  \sum_{\ell=0}^{L} q_\ell (\sigma_{\ell+1}^2 \alpha_\ell \gamma_\ell )^\ell =  q_L \sum_{\ell=0}^{L} (\sigma_{L+1}^2 \alpha_{L} \gamma_{L})^\ell $ by \cref{assumption:simplify}. 
Set $x_L = L(1 - \sigma_{L+1}^2 \alpha_L \gamma_L)$. We have $\lim_{L \to \infty}x_L = \eps_1 + \eps_2$.
Then  $L^{-1}\sum_{\ell=0}^{L} (\sigma_{L+1}^2 \alpha_L \gamma_L)^\ell =   x_L^{-1}[ 1- (1 - x_L/L)^L] \to (\eps_1 + \eps_2)^{-1}( 1 - \exp(-\eps_1 - \eps_2))$. Then the assertion has been proven.
\end{proof}

Next, consider the usual FIM.
Fix $N \in \N$ and consider input vectors $x(1), \dots, x(N) \in \R^M$.
Set $\post^0(n)=x(n)$.
Since $\mc{I}(\theta) =  N^{-1}\sum_{n=1}^N \mc{I}(\theta | x(n))$, the FIM $\mc{I}(\theta)$ shares non-zero eigenvalues with the dual $\Theta$, which is the $N \times N$ matrix whose $(m,n)$-entry is given by the following $M \times M$ matrix:
\begin{align}
    \Theta(m,n) &= \frac{1}{N} \pp{ f_\theta(x(m))}{\theta} \pp{f_\theta(x(n))}{\theta}^\top,
\end{align}
where $m,n=1, \dots, N$.
For $\ell=1, \dots, L$, set recursively $\pre^\ell(n) = W_\ell \post^{\ell-1}(n)$, $\post^\ell(n) = \act^\ell(\pre^\ell(n) )$, and define $\delta_{L \to \ell}(n)$ in the same way.
Then $\Theta(m,n) =  MN^{-1} \sum_{\ell=1}^L \Sigma_\ell(m,n) \delta_{L \to \ell}(m) \delta_{L \to \ell}(n)^\top$, where $\Sigma_\ell(m,n) =  M^{-1}\sum_{i=1}^M x_{\ell,i}(m)x_{\ell,i}(n)$, and we have the following block-matrix representation:

\begin{align}
    \Theta= \frac{M}{N} \begin{bmatrix} H_L(x(1)) & * & \dots  & * \\
    * & H_L(x(2)) &  \cdots  &  * \\
    \vdots & \cdots & \ddots & \vdots \\
    * & * & \cdots  &  H_L(x(N))
    \end{bmatrix}
\end{align}
Hence for the $N$-sample, considering the collection of eigenvalues of the dual conditional FIMs $(H_L(x(n)))_{n=1}^N$ is equivalent to considering block-diagonal approximation of the (scaled)  dual FIM $ MN^{-1} \Theta$. 
Even if it is not clear yet that the block-diagonal approximation behaves well,   the mean of eigenvalues of full matrix is exactly determined by the diagonal part as the following assertion.
\begin{cor} \label{cor:mean-full}
Denote by $m_{L,N}$ the wide limit  $M \to \infty$ of  the mean of eigenvalues of $\Theta/M$.
Then under the limit $L, N \to \infty$ with $L/N \to \alpha  < \infty$,   it holds that 
\begin{align}
 m_{L,N} \to \alpha q \frac{1-  \exp(-\eps_1 - \eps_2)}{\eps_1 + \eps_2}.
\end{align}
\end{cor}
\begin{proof}
Fix $L, N$. The mean of eigenvalues of $\Theta/M$ is equal to $\mathrm{tr}(\Theta/M)$ and 
$\mathrm{tr}(\Theta/M) =  \sum_{n=1}^N \mathrm{tr}(H_L(x(n)))/N^2 \to m_1(\mu_L)/N$. By \cref{prop:expectation}, the assertion follows.
\end{proof}
\cref{cor:mean-full} implies that the mean eigenvalue of $\Theta/M$ is close to $qLN^{-1}$ when $L$ and $N$  are of the same magnitude.
In the next section, we empirically examine the block-diagonal approximation.

\section{Empirical Analysis}
\subsection{Expected vs Conditional FIM}\label{sssec:expected_vs_conditional}

\begin{figure}[htbp]
    \centering
    \includegraphics[width=0.49\linewidth]{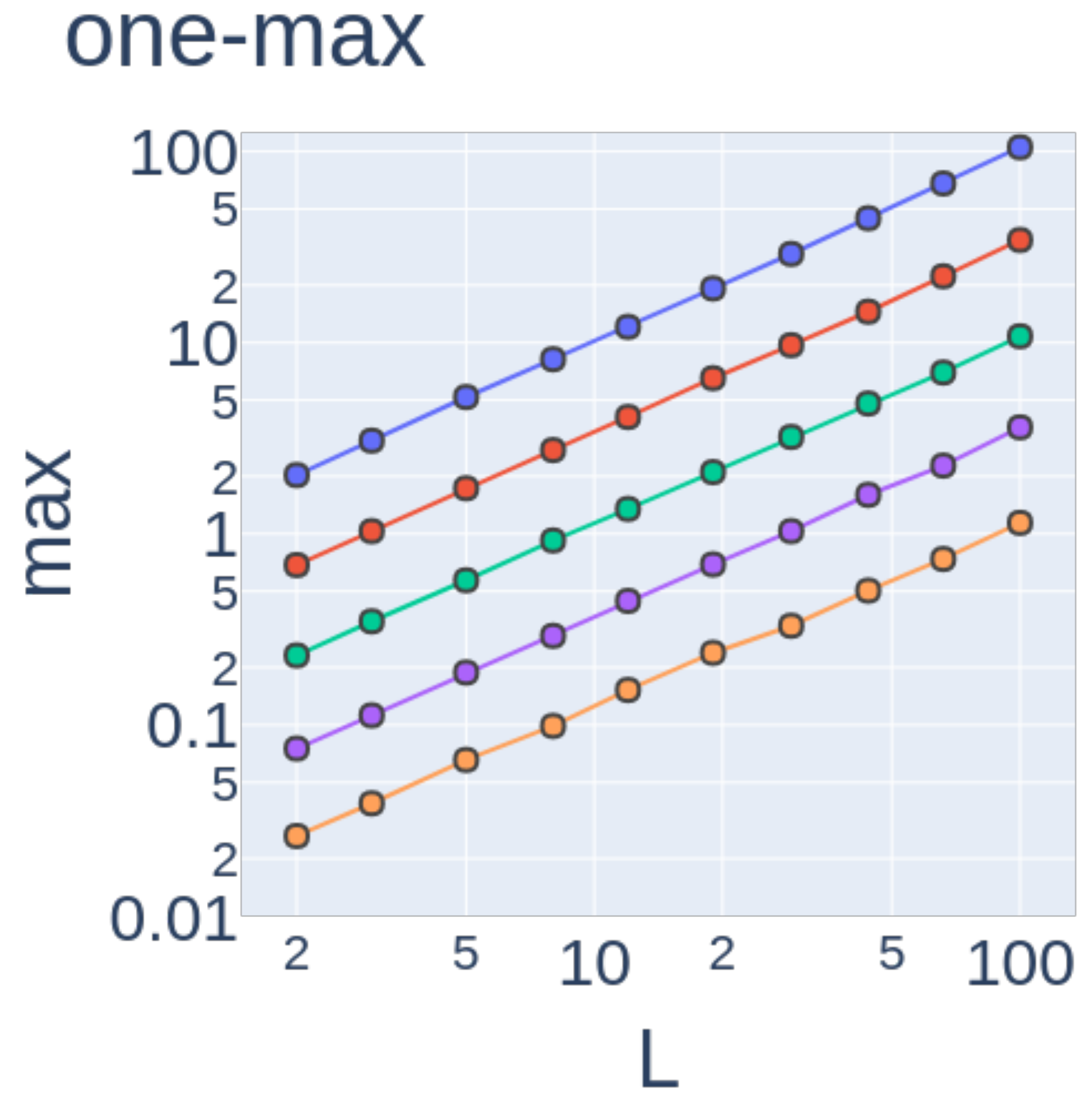}
    \includegraphics[width=0.49\linewidth]{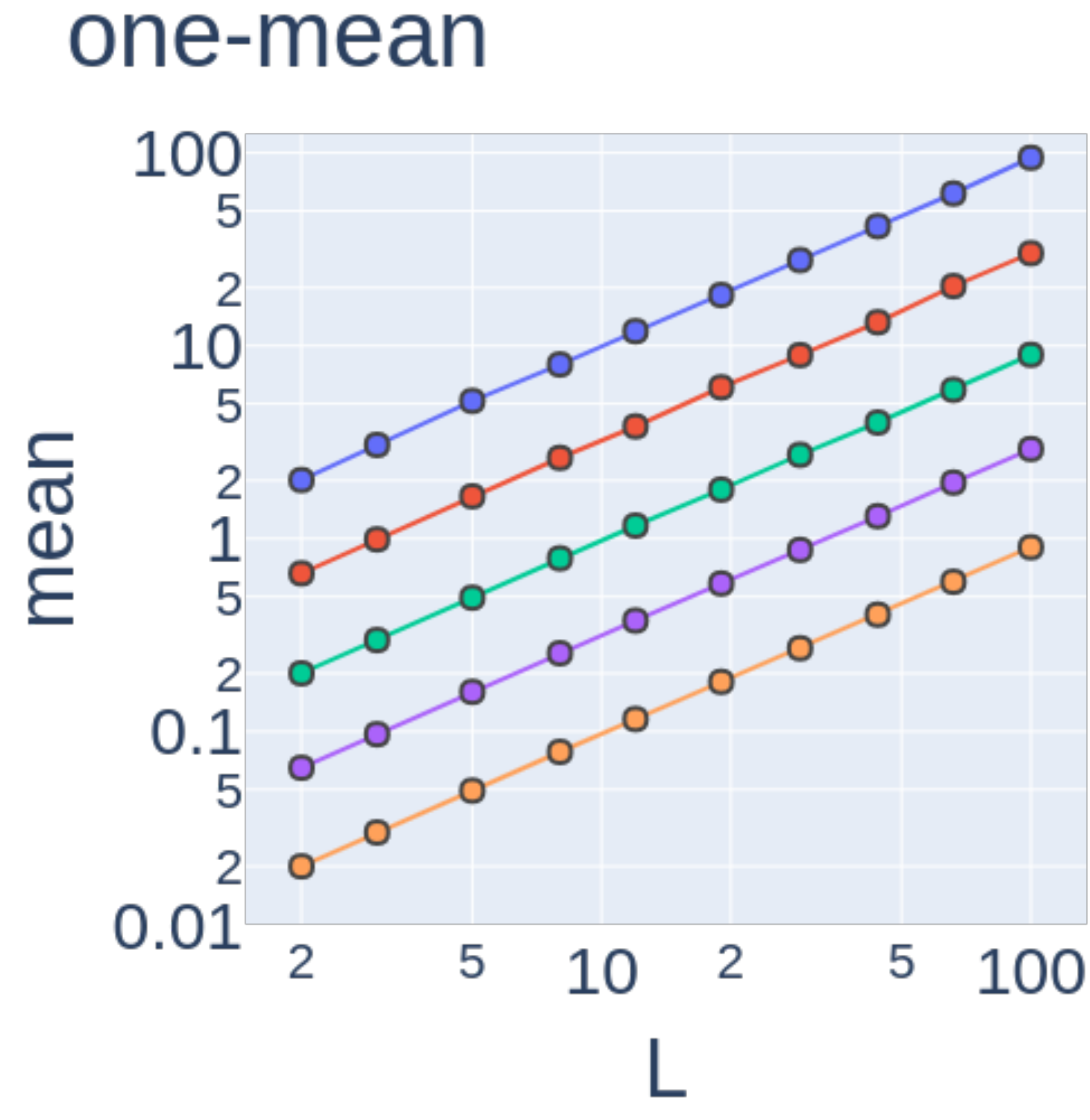}
    \includegraphics[width=0.49\linewidth]{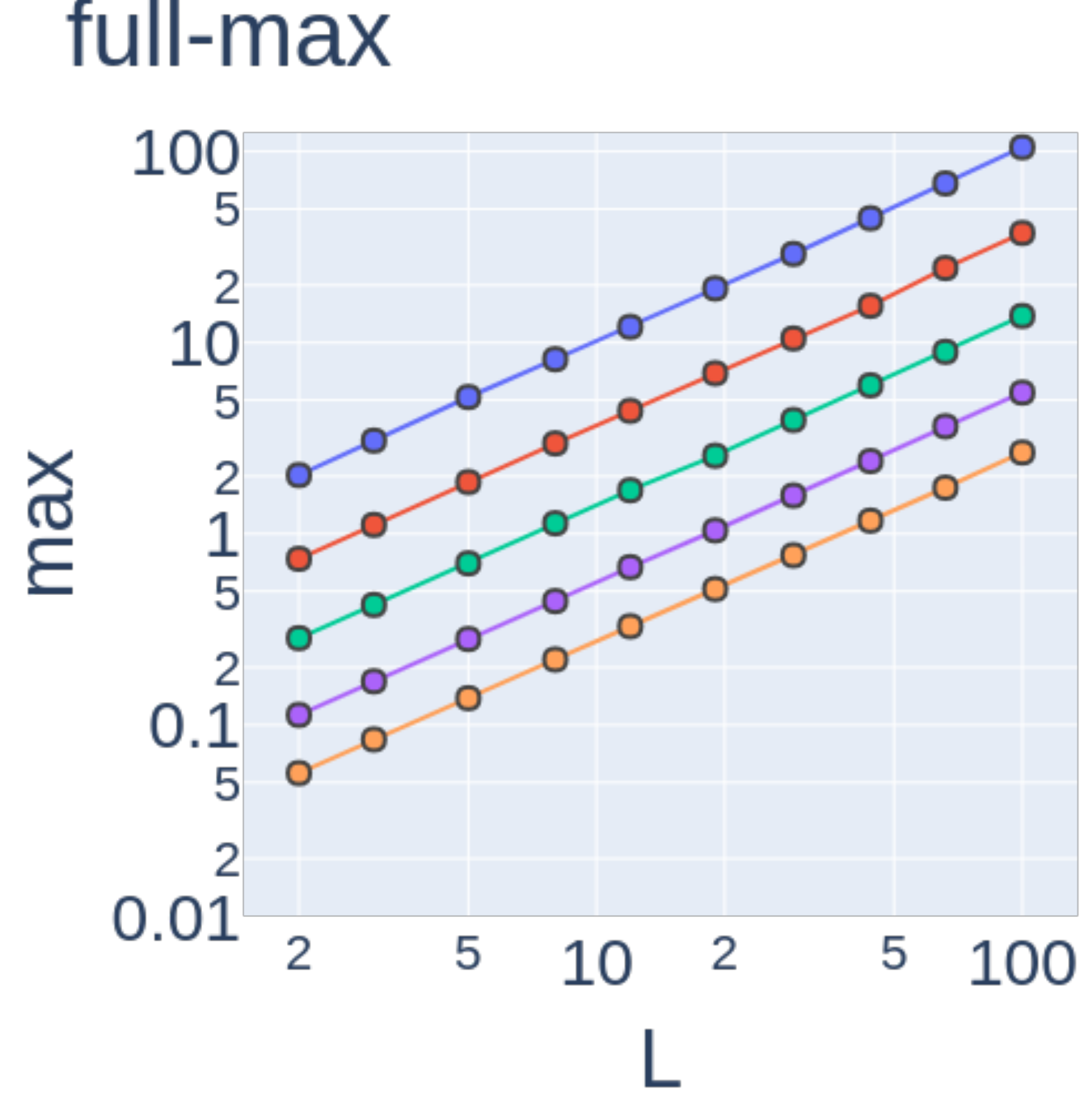}
    \includegraphics[width=0.49\linewidth]{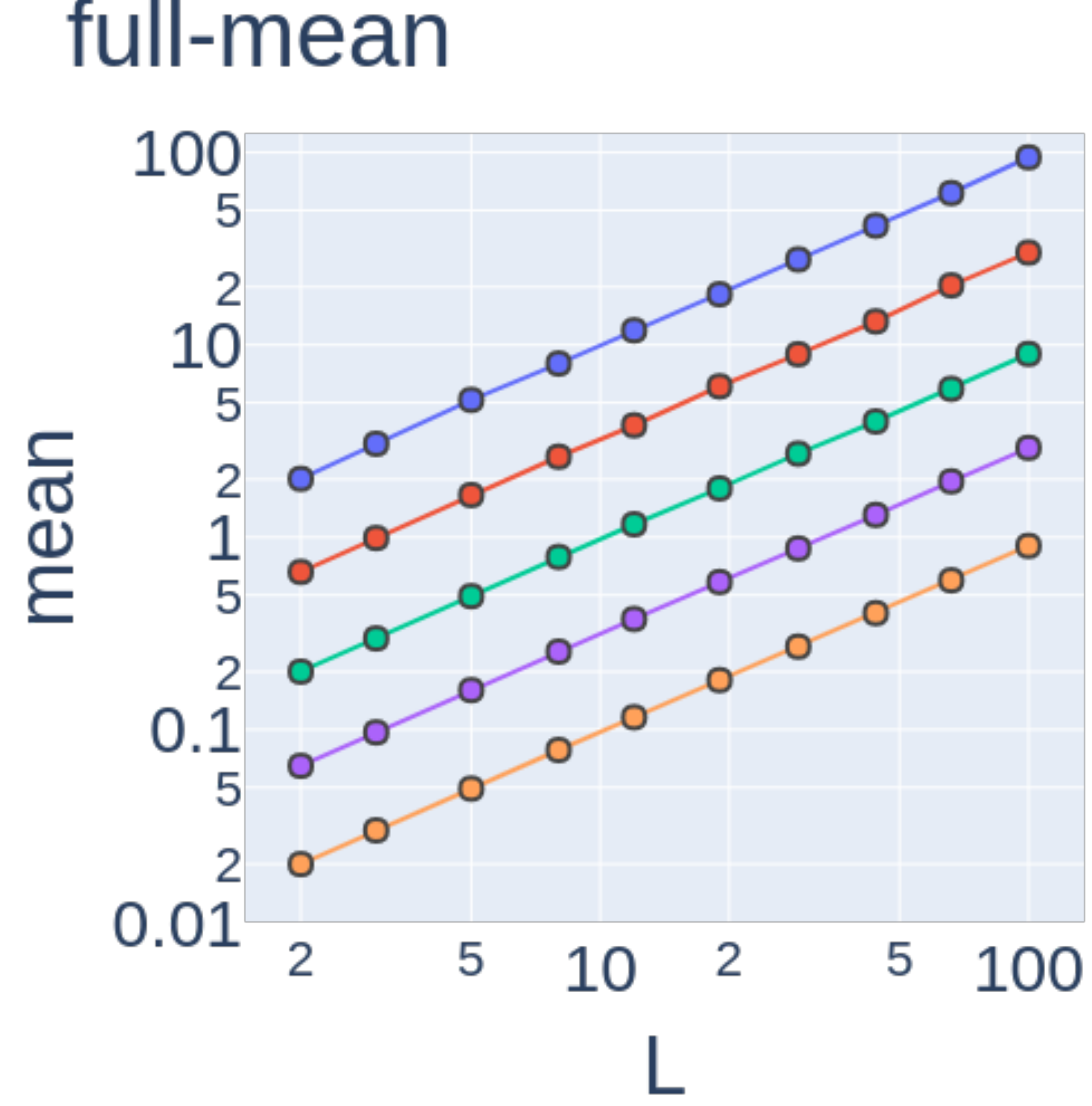}
    \caption{Eigenvalues of $\Theta/M$ (full matrix) and the collection of eigenvalues of $(H_L(x(n))/N)_{n=1}^N$ (one block); We show maximum  and mean in each case. We set $M=100$ and each axis logarithmic. }
    \label{fig:stats_of_eigs}
\end{figure}

\begin{comment}
\begin{wrapfigure}{R}{0.30\textwidth}
    \centering
    \vspace{-10pt}
    \includegraphics[keepaspectratio,width=0.30\textwidth]{}
    %\includegraphics[width=0.49\linewidth]{silu/L-20_M-784_N-20_htanh_full_vs_one.png}
        %\vspace{-10pt}
    \caption{Histograms of eigenvalues. }
    \vspace{-20pt}
    \label{fig:full_vs_one}
\end{wrapfigure}
\end{comment}

\begin{figure}[ht]
    \centering
    \includegraphics[width=0.62\linewidth]{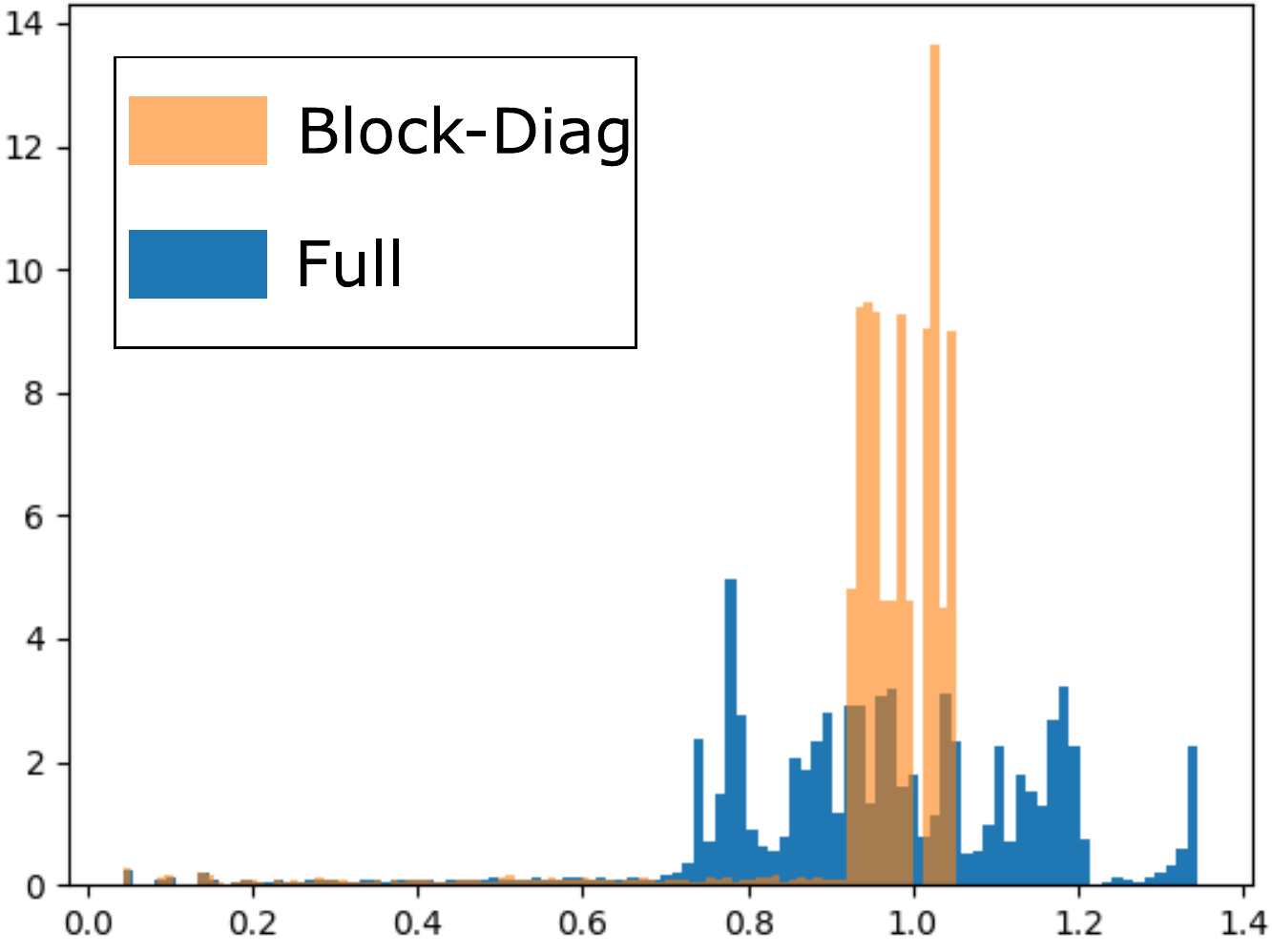}
    \caption{Histograms of eigenvalues. }%\com{REPLACE FONT: truefont to Type 1}
    \label{fig:full_vs_one}
\end{figure}

In order to investigate the difference between the usual FIM \eqref{align:fim} for multi-samples and the conditional FIM \eqref{align:cfim},  we numerically computed eigenvalues of their dual matrices with $\hat{q}_0 =1$.
\cref{fig:stats_of_eigs} shows the statistics of their eigenvalues of $\Theta/M$ (Full) and $(H_L(x(n))/N)_{n=1}^N$ (Block-Diag) with $L=N=10$.
We observed that the maximum and the mean eigenvalue of each matrix are $O(L/N)$, except for the maximum eigenvalue of $\Theta/M$  with a large $N$.

\cref{fig:full_vs_one} shows an example of  eigenvalues of $\Theta/M$ (Full) and $(H_L(x(n))/N)_{n=1}^N$ (Block-Diag) concentrated around $L/N$ for a  small $N$.  For the conditional FIM, we observed theoretical predictions \cref{prop:expectation} and  \cref{thm:maximum} agree well with experimental results.

\subsection{Training Dynamics}\label{ssec:dynamics}

\begin{figure*}[t]
   \centering
    \includegraphics[width=0.2452\linewidth]{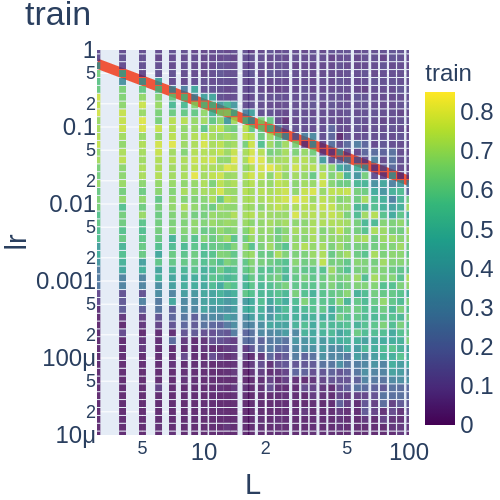}
    \includegraphics[width=0.2452\linewidth]{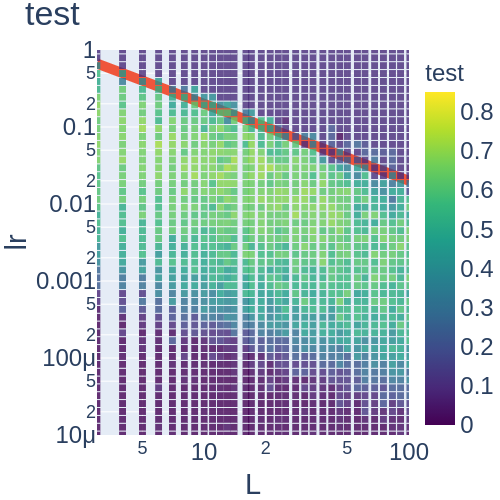}
    \includegraphics[width=0.2452\linewidth]{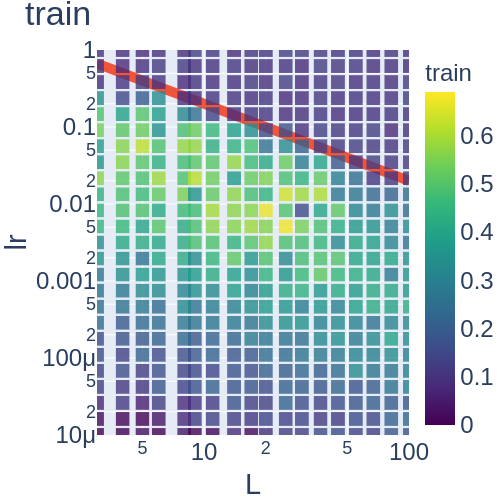}
    \includegraphics[width=0.2452\linewidth]{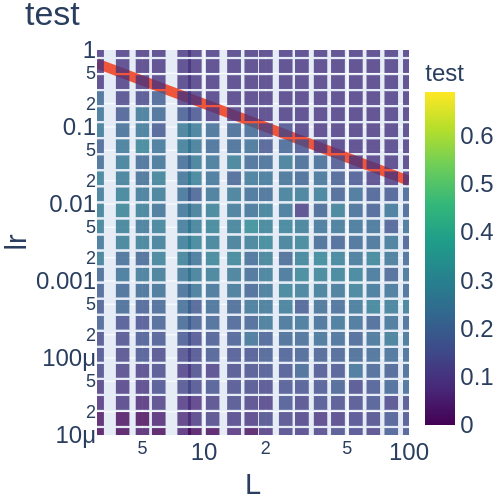}

    \caption{
Accuracy heatmaps for different value of $L$ (x-axis) and lr ($=\eta$) (y-axis) after online training. 
In each figure, the line is $\eta=2/L$. 
Each axis is logarithmic. 
Each network was trained on Fashion-MNIST or CIFAR10. 
In each experiment, the training dataset consists of 500 data points sampled uniformly from a whole data set, and the testing dataset consists of 10000 samples separated from the training dataset.
For the Fashion-MNIST (resp.\,CIFAR10), the network was trained for a single epoch (resp.\,ten epoch) training.
The leftmost figure (resp.\,the second from left) is the heatmap of the training (resp.\,testing) accuracy evaluated after training on Fashion-MNIST.
 The third figure from left (resp.\,the rightmost ) is the heatmap of the training (resp.\,testing) accuracy evaluated after training on CIFAR10.
    }
\label{fig:L_vs_lr_N1}
\end{figure*}

We investigate how our FIM's spectrum affects  training dynamics.
Consider the online gradient descent method: 
\begin{align}
\theta_{t+1} = \theta_t - \eta \nabla_\theta [M^{-1}\mathcal{L}(f_{\theta_t}(x(t) - y(t))].    
\end{align}
Under the first order Tayler approximation of $f_\theta$ around the initial parameter $\theta_0$,  we have the following approximation:
\begin{align}
&\theta_{t+1} \sim  (I - \eta  H_L(x(t), \theta_0) ) \theta_t \notag \\ 
&+ M^{-1} f_{\theta_0}(x)^\top (y- f_{\theta_0}(x) + \nabla_\theta f_{\theta_0}(x(t))^\top\theta_0 ).    
\end{align}
Hence at the initial phase of the training, the condition $||I - \eta  H_L||< 1$ is necessary to avoid the explosion of parameters. In particular, 
\begin{align}\label{align:boundary}
\eta  <  \frac{2}{\lambda_{\max}(H_L)}.
\end{align}
%\com{ For the steepest gradient descent method, the similar discussion of the bound using maximal eigenvalue is given in \cite{Cun1991eigenvalues}.}
By  \cref{thm:maximum}, we expect that the boundary $2/\lambda_{\max}(H_L)$ will be close to $2/(qL)$.

\begin{comment}
\begin{equation}
   f_\theta(x) \sim   f_{\theta(0)}(x) + \nabla_\theta f_{\theta(0)}(x)^\top(\theta-\theta(0)) 
\end{equation}
It holds that
\begin{align}
\theta(t+1) &= \theta(t) - \eta \nabla_\theta f_{\theta(0)}(x(t)) (y- f_{\theta(0)}(x) - \nabla_\theta f_{\theta(0)}(x))^\top(\theta(t)-\theta(0)) ) \\ 
&=  (I - \eta \nabla_\theta f_{\theta(0)}(x) f_{\theta(0)}(x)^\top) \theta(t)  \\ &+f_{\theta(0)}(x)^\top (y- f_{\theta(0)}(x) + \nabla_\theta f_{\theta(0)}(x)^\top\theta(0)  )
\end{align}
\end{comment}

\begin{comment}

Furthermore, consider the infinitesimal version :
\begin{align}
  \frac{ d\theta(t)}{dt} =  - \eta \nabla_\theta E(\theta(t))
\end{align}

\begin{align}
    \Theta(t) := \frac{1}{N} \pp{f_{\theta(t)}}{\theta}\pp{f_{\theta(t)}}{\theta}^T.
\end{align}

Let $u(t) = (f_{\theta(t)}(x_1) - y_1, \dots, f_{\theta(t)}(x_N) - y_N)$.
Now 
\begin{align}
\frac{d u(t)}{dt} =  \pp{f_{\theta(t)}}{\theta}   \frac{ d\theta(t)}{dt} =  - \eta \pp{f(\theta(t))}{\theta}   \nabla_\theta E(\theta(t))
=   - \eta \Theta(t) u(t)
\end{align}
\end{comment}

%Hence if we use the optimal learning rate, the bottleneck of the convergence speed is controlled by the condition number $\lambda_{\min}(\Theta)/\lambda_{\max}(\Theta)$.

%\subsection{Experiments}
To confirm the boundary, we exhaustively searched training and testing accuracy while changing $L$ and $\eta$ with normalizing inputs with $\hat{q}_0 = 1$.
In \cref{fig:L_vs_lr_N1},  we observed the theoretical prediction $\eta = 2/L$ by \eqref{align:boundary} coincided nicely with the boundary of the collapse of the accuracy obtained in experiments.

We normalized each input so that $\hat{q_0}=1$ and converted class labels to an orthonormal system in $\R^{M}$.
In whole experiments, we commonly use the hard-tanh activation with $s^2=0.125$ and $g=1.0013$  to archive dynamical isometry. %\cite{Pennington2018emergence}. 
After training, we computed the MSE loss $\mc{L}/M$ and accuracy on the dataset of $10000$ samples separated from the dataset for training.

We trained the network on benchmark datasets Fashion-MNIST  \citep{Xiao2017fashion} 
and CIFAR10 \citep{krizhevsky2009learning}. 
The Fashion-MNIST (resp.\,CIFAR10) consists of $28 \times 28$ (resp.\,$3\times32\times32$) dimensional images and $10$ class labels,
We applied the online gradient descent on $500$ data, which is uniformly sampled from whole data and fixed, for an epoch (resp.\, ten epochs) in the Fashion-MNIST (resp.\,CIFAR10).
In \cref{fig:L_vs_lr_N1},  the difference between training and testing accuracy on the CIFAR10 was larger than that on the Fashion-MNIST. 
The reason for this is because of the overfitting of DNNs to the small dataset consists of only 500 data. 
However, agreement with the theoretical line was also visible in the testing accuracy on the CIFAR10.

Additionally, we observed that the testing loss slightly violated the boundary at large $L$ (i.e., $\eta > 2/L$ and the test loss was not large). (See the supplemental material B for the detail.)
In this region, we found that the spectral distribution of $H_L$ during training was far different from that in the initial state. Thus, although the parameters did not explode, we had a qualitative change of the optimization in the seeping region. In this sense, the theoretical boundary explains well the state of training.

\begin{comment}
\begin{figure}[htbp]
    \centering
   % \includegraphics[width=0.49\linewidth]{N10/mean_train_loss.png}
    \includegraphics[width=0.49\linewidth]{N10/test_loss.png}
    \includegraphics[width=0.49\linewidth]{N10/test_accuracy.png}
    \caption{Heatmaps of loss or accuracy for different L and lr, for $N=10$}.
    \label{fig:L_vs_lr_N10}
\end{figure}

\begin{figure}[htbp]
    \centering
    %\includegraphics[width=0.49\linewidth]{N100/mean_train_loss.png}
    \includegraphics[width=0.49\linewidth]{N100/test_loss.png}
    \includegraphics[width=0.49\linewidth]{N100/test_accuracy.png}
    \caption{Heatmaps of loss or accuracy for different L and lr, for $N=100$}.
    \label{fig:L_vs_lr_N10}
\end{figure}

\begin{figure}[htbp]
    \centering
    \includegraphics[width=0.49\linewidth]{LvsN_loss.png}
    \includegraphics[width=0.49\linewidth]{LvsN_acc.png}
    \caption{L vs N }
    \label{fig:LvsN}
\end{figure}

\end{comment}

\section{Conclusion}
Our study establishes a springboard for a new way to examine Fisher information of neural networks with a powerful methodology provided by the free probability theory. 
In particular, we have shown that the dual conditional FIM's eigenvalues get concentrated at the maximum when DNNs achieve approximately dynamical isometry. 
Furthermore, we have explicitly solved a propagation equation and shown the spectrum distribution's exact form in the case depth is three. 

As the study's evidence indicates, the empirical Fisher information matrix is in  $L/N$ growth rate.
Interestingly, a depth-dependent learning rate has been empirically observed in  DNNs achieving dynamical isometry \citep{Pennington2017Resurrecting} and required for the convergence of training in deep linear networks \cite{hu2020provable}. Our theoretical results support their empirical observation and clarify the intrinsic distortion of the parameter space.

We are aware that our spectral analysis of the FIM may have a few limitations.
A limitation is that our analysis is based on the asymptotic freeness of Jacobi matrices (\cref{assmptn:freeness}).
We expect that the works \citep{Hanin2019products, Yang2019scaling, Pastur2020random}, which prove or treat the asymptotic freeness with Gaussian initialization, will help us to prove it with the orthogonal initialization.
Another limitation is that the batch size is limited to be small in our theory. 
Future analysis of block random matrices via free probability theory will investigate the full Fisher information matrix spectrum.

\section{Acknowledgement}
The authors gratefully acknowledge valuable and profound comments of Roland Speicher.
The authors would like to thank Hiroaki Yoshida and Noriyoshi Sakuma for constructive discussion.
TH acknowledges the funding support from  JST ACT-X Grant number JPMJAX190N. 
RK acknowledges the funding support from JST ACT-X Grant Number JPMJAX190A.

\bibliographystyle{unsrtnat}
%\bibliographystyle{abbrv}
%\bibliography{references.bib}
\bibliography{references_direct.bib}

%\end{comment}

%\begin{comment}
%%%
%%% Supplemental material
%%%
\onecolumn
\numberwithin{equation}{section}
\appendix
\section{Proof for The Two-Hidden-Layer Case}
Here we provide the postponed proof.
\begin{proof}
Set $\tilde{\nu}_\ell = ( \gamma_\ell^{-1}\,\cdot\,)_* \nu_\ell =  \alpha_\ell \delta_1 + (1-\alpha_\ell)\delta_0$.
Then we have 
\begin{align}
    \mu_3  = \affine{ \sigma^2_3\gamma_2 }{ q_2}  \left[ \tilde{\nu}_2 \boxtimes \affine{ \sigma_2^2 \gamma_1 q_0 }{q_1} \tilde{\nu}_1   \right].  
\end{align}

By replacing  $\sigma_{\ell+1}^2\gamma_\ell q_{\ell-1}q_{\ell}^{-1}$ with $\gamma_\ell$ ($\ell=1,2$), we may assume that  $q_\ell = \sigma_\ell = 1$.
Write
%\begin{align}
 $   \xi = \affine{\gamma_1}{1} \tilde{\nu}_1.$
%\end{align}
Since $S_{\tilde{\nu}_2 \boxtimes \xi}(z) = S_{\tilde{\nu}_2 }(z) S_{\xi}(z)$, 
we have $h_\xi^{-1}(z) =   h_{\tilde{\nu}_2 \boxtimes \xi}^{-1}(z)S_{\tilde{\nu}_2}(z)$. 
Hence $h_{\tilde{\nu}_2\boxtimes\xi}(z)$ is the solution  of the following equation  on $w$:
 \begin{align}\label{align:equation}
 w = h_{\xi}( zS_{\tilde{\nu}_2 }(w) ) .
\end{align}
Note that
\begin{align}
 S_{\tilde{\nu}_2}(z) = \frac{z+1}{z+ \alpha_2},   \
 h_\xi(z)  =  z \left[\frac{1-\alpha_1}{z-1}+ \frac{\alpha_1}{z-1-\gamma_1} \right]-1 .
\end{align}
Thus the solution of \eqref{align:equation} is given by
\begin{align}
w= \frac{   g(z) \pm z\sqrt{-f(z)}  }{ 2(z-1)(1+\gamma_1-z)},
\end{align}
where
%\begin{align}
$f(z) =  ( \lambda_+ -z  )(z -\lambda_-),$
$\lambda_\pm   =    1 + \gamma_1 \left( \sqrt{\alpha_1(1 - \alpha_2)} \pm\sqrt{ \alpha_2(1 - \alpha_1)} \right)^2$,
and $g(z) = (z - 1)(z - 2(1+\gamma_1)\alpha_2) - \gamma_1(\alpha_1 -\alpha_2)z.$
%\end{align}
By  $G(z) = (h(z) +1 )/z$ and by the condition $\Im G(z) < 0 $ if $\Im z > 0$, we have 
\begin{align}
    G_{\tilde{\nu}_2 \boxtimes \xi}(z) =  \frac{1}{z}\left[ 1 +  \frac{ g(z)   }{ 2(z-1)(1+\gamma_1-z)}\right]
    + \frac{ \sqrt{-f(z)}  }{ 2(z-1)(1+\gamma_1 - z)}.
\end{align}
Note that $1 \leq \lambda_- \leq \lambda_+ \leq 1 + \gamma_1$.
By the Stieltjes inversion, the absolutely continuous part of $\nu \boxtimes \mu$ is 
\begin{align}\label{align:absconti}
 -\frac{1}{\pi} \lim_{y \to +0}\Im G_{{\tilde{\nu}_2 \boxtimes \xi}}(x+ y\sqrt{-1})  =\frac{\sqrt{f(x)}}{2\pi(x-1)(1+\gamma_1-x)}{\bf 1}_{\{f \geq 0\}}(x)  \  (x \in \R).
\end{align}
Te weights of the atoms are given by
\begin{align}
\lim_{y \to +0} z G_{\tilde{\nu}_2 \boxtimes \xi}(y\sqrt{-1}) &= 1 - \alpha_2,\label{align:purepoint1} \\
\lim_{y \to +0}  (z-1) G_{\tilde{\nu}_2 \boxtimes \xi}(1+ y\sqrt{-1} ) &=  (\alpha_2- \alpha_1)^+,\label{align:purepoint2}\\
\lim_{y \to +0}(z -1-\gamma_1 )G_{\tilde{\nu}_2 \boxtimes \xi}(1 + \gamma_1 + y\sqrt{-1}) & =  ( \alpha_1 + \alpha_2 -1  )^+,\label{align:purepoint3}
\end{align}
where $a^+ = \max(a,0)$ for $a \in \R$.
By \cite{Belinschi2003atoms}, the free multiplicative convolution $\tilde{\nu}_2 \boxtimes \xi$ has no singular continuous part. Hence $\tilde{\nu}_2 \boxtimes \xi$ is the sum of the absolutely continuous part \eqref{align:absconti} and the pure point part (\ref{align:purepoint1}, \ref{align:purepoint2}, \ref{align:purepoint3}) as follows:
\begin{align}
 ( \tilde{\nu}_2 \boxtimes \xi) (dx)& =   (1- \alpha_2)\delta_0(dx) + (\alpha_2 - \alpha_1)^+ \delta_1(dx)  
  +  (\alpha_1 + \alpha_2 -1)^+\delta_{1+ \gamma_1}(dx)\\
 & +  \frac{\sqrt{(\lambda_+ -x)(x-\lambda_-)}}{2\pi(x-1)( 1+  \gamma_1-x)}{\bf 1}_{[\lambda_-, \lambda_+]}(x)(dx).
\end{align}
It holds that  $\mu_{3} = \affine{\gamma_2}{1}(\tilde{\nu}_2 \boxtimes  \xi)$.
We have completed the proof.
\begin{comment}
Lastly, consider arbitrary  $\sigma_\ell$ and $q_\ell$.  By replacing  $\gamma_\ell$ with $\sigma_{\ell+1}^2\gamma_\ell q_{\ell-1}q_{\ell}^{-1}$ in the above discussion, we have completed the proof.
Next, consider arbitrary  $\sigma_\ell$ and $q_\ell$.  By replacing  $\gamma_\ell$ with $\sigma_{\ell+1}^2\gamma_\ell q_{\ell-1}q_{\ell}^{-1}$ in the above discussion, we have $\mu_3(dx) = \mu_\mr{atoms}(dx) + \rho(x)dx$, where
  \begin{align} 
\mu_\mr{atoms} &= (1- \alpha_2)\delta_{\lambda_{\min}} 
+ (\alpha_2 - \alpha_1)^+ \delta_{\lambda_\mathrm{mid}} 
+   (\alpha_1 + \alpha_2 -1)^+ \delta_{ \lambda_{\max} }(dx), \\
\rho(x) &=  \frac{\sqrt{  (\lambda_+ -x)(x - \lambda_- ) }}{2\pi (x- \lambda_\mathrm{mid}  )(  \lambda_{\max}  - x)}{\bf 1}_{[\lambda_-, \lambda_+] }(x), 
\end{align}
with 
\begin{align}
    \lambda_{\min} &= q_2,\\
    \lambda_\mathrm{mid} &= q_2 +\sigma_3^2\gamma_2 q_1, \\
    \lambda_\pm   &=  q_2  +  \sigma_3^2\gamma_2 \left(  q_1 +   \sigma_2^2\gamma_1 q_0\left( \sqrt{\alpha_1(1 - \alpha_2)} \pm\sqrt{ \alpha_2(1 - \alpha_1)} \right)^2 \right), \text{\ and}\\
    \lambda_{\max} &= q_2 +   \sigma_3^2\gamma_2( q_1  +   \sigma_2^2\gamma_1 q_0).
\end{align}
We have completed the proof.
\end{comment}
\end{proof}

\begin{comment}

\section{Additional Notation}
\paragraph{Pushforward}
Given $a, b \in \R$ with $a \neq 0$  and a Borel probability distribution $\nu$ on $\R$, 
the pushforward of $\nu$ by the affine transformation $x \mapsto ax + b$ is defined as the Borel probability distribution $\affine{a}{b}\nu$ given by 
%\begin{align}
 $ \affine{a}{b}\nu(B) = \nu( a^{-1}(B - b))   \text{\ for Borel set \ $B \subset \R$}$.
%\end{align}
In the case $a = 1$, we simply write the pushforward as $(+b)_*$.
%\begin{comment}
For example, it holds that 
\begin{align}
\affine{a}{b}\delta_\gamma = \delta_{b + a\gamma },
\end{align} 
where $\gamma \in \R$ and $\delta_\gamma$ is the delta probability distribution supported on the atom $\{ \gamma \}$.
\end{comment}

\section{Out of  Initial Spectral Distribution}
\cref{fig:L_vs_lr_AL} shows the detailed results of experiments in Section 5.2.
As discussed in Section 5.2, the accuracy followed the boundary (upper figures).
However, in the boundary areas with $L/2< \eta < 0.2$ the loss reduced (lower figures), but the accuracy did not improve.
Hence the loss violates the boundary $\eta = L/2$ predicted by Theorem 4.3 and (35), but the accuracy does not.
In this section, we discuss the region $L/2< \eta < 0.2$ of the loss heatmaps.

Recall that we trained the network on benchmark datasets Fashion-MNIST \citep{Xiao2017fashion} 
and CIFAR10 \citep{krizhevsky2009learning}. 
The Fashion-MNIST (resp.\,CIFAR10) consists of $28 \times 28$ (resp.\,$3\times32\times32$) dimensional images and $10$ class labels.
We set  $M=28^2$ for the Fashion-MNIST  and $M=32^2$ for the CIFAR10, by shrinking the first layer in the case of the CIFAR10.
We applied the online gradient descent on $500$ data, which is uniformly sampled from whole data and fixed, for an epoch (resp.\, ten epochs) in the Fashion-MNIST (resp.\,CIFAR10).
Recall that we normalized each input so that $\hat{q_0}=1$ and converted class labels to an orthonormal system in $\R^{M}$, and use the hard-tanh activation with $s^2=0.125$ and $g=1.0013$  to archive dynamical isometry. 
After training, we computed the average of MSE loss  for each dataset $\mathcal{D}$,   which is given by the following:
\begin{align}
    \mathrm{Loss}(x,l) =   \frac{1}{2M |\mathcal{D}|} \sum_{(x,\ell) \in \mathcal{D}}{||f_\theta(x) - e_\ell||^2_2},
\end{align}
where $e_m \in \R^M$ ($m=1, 2, \dots, M$) is the unit vector whose $\ell$-th entry is one and the other entries are zero, and 
the dataset $\mathcal{D}$ is the training dataset, which consists of $500$-samples, or the testing dataset of $10000$ samples separated from the training dataset.
We also computed top-1 accuracy.
In \cref{fig:L_vs_lr_AL}, the difference between training and testing accuracy on the CIFAR10 was larger than that on the Fashion-MNIST. 
The reason for this is because of the overfitting of DNNs to the small dataset consists of only 500 data. 
However, agreement with the theoretical line was also visible in the testing accuracy on the CIFAR10.

\begin{figure*}[h]
   \centering
    \includegraphics[width=0.245\linewidth]{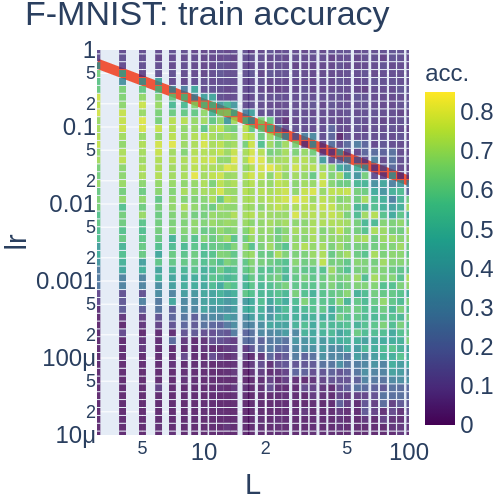}
    \includegraphics[width=0.245\linewidth]{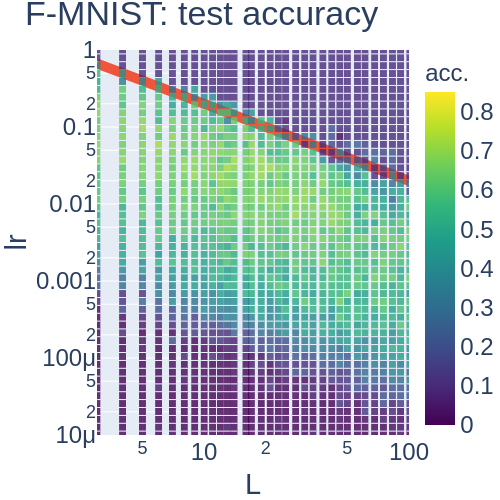}
    \includegraphics[width=0.245\linewidth]{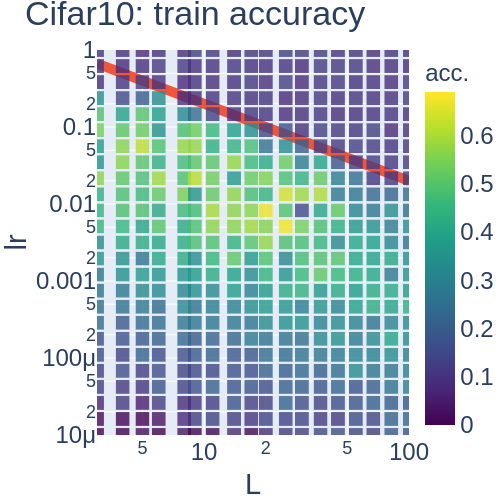}
    \includegraphics[width=0.245\linewidth]{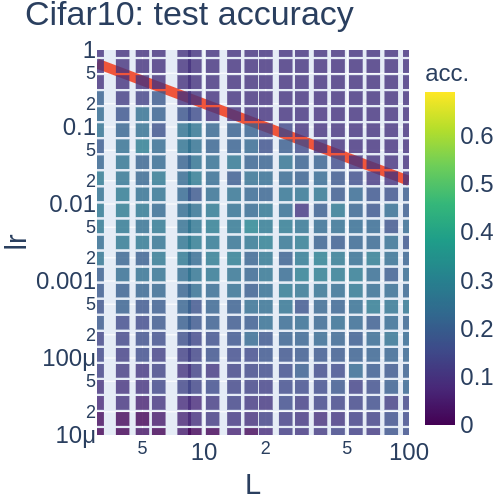}

    \includegraphics[width=0.245\linewidth]{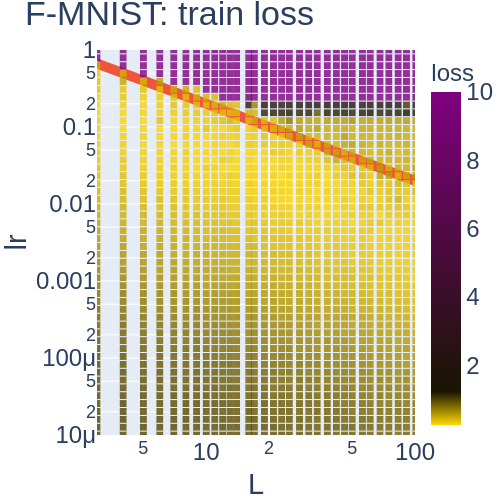}
    \includegraphics[width=0.245\linewidth]{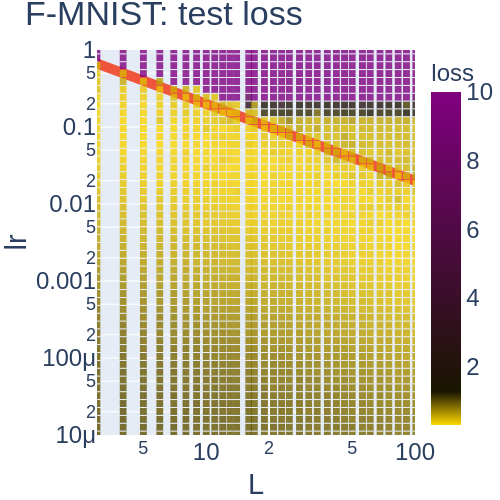}
    \includegraphics[width=0.245\linewidth]{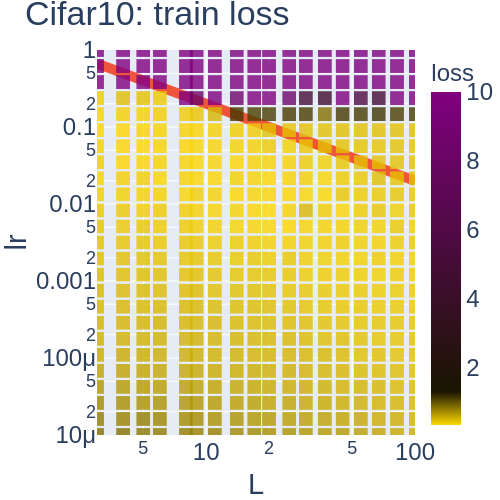}
    \includegraphics[width=0.245\linewidth]{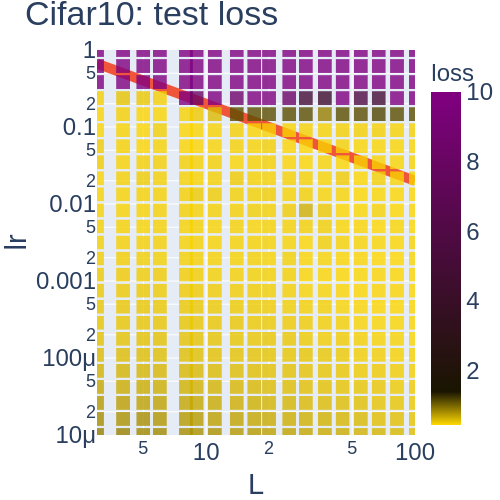}

    \caption{Accuracy heatmaps (upper figures) and loss heatmaps (lower ones) for different value of $L$ (x-axis) and lr ($=\eta$) (y-axis) after online training. 
    In each figure, the line is $\eta=2/L$. 
    Each axis is logarithmic. 
    Each network was trained on Fashion-MNIST or CIFAR10. 
    In each experiment, the training dataset consists of 500 data points sampled uniformly from a whole data set, and the testing dataset consists of 10000 samples separated from the training dataset.
    For the Fashion-MNIST (resp.\,CIFAR10), the network was trained for a single epoch (resp.\,ten epoch) training.
    In each figure of losses, we have rounded off any losses above 10 and show them as 10.
    }
    \label{fig:L_vs_lr_AL}
\end{figure*}

We focus on areas near the boundary  $\eta=2/L$ at large $L$ and we show how different the conditional FIM $H_L$'s spectral distribution during training was from that in the initial state.
\cref{fig:esd_dFIM} shows that eigenvalue distributions of $H_L$ near the boundary at large $L$.
We observed that most of the eigenvalue distributions shrunk in the areas with $\eta > 2/L$.
The shrinkage in the eigenvalue distribution is the reason for the reduction in test loss in the areas.

However, the reason why the appearance of the second boundary around $\eta =0.2$ appeared has not been revealed yet.

\begin{figure}[hb]
    \centering
    \includegraphics[width=0.785\linewidth]{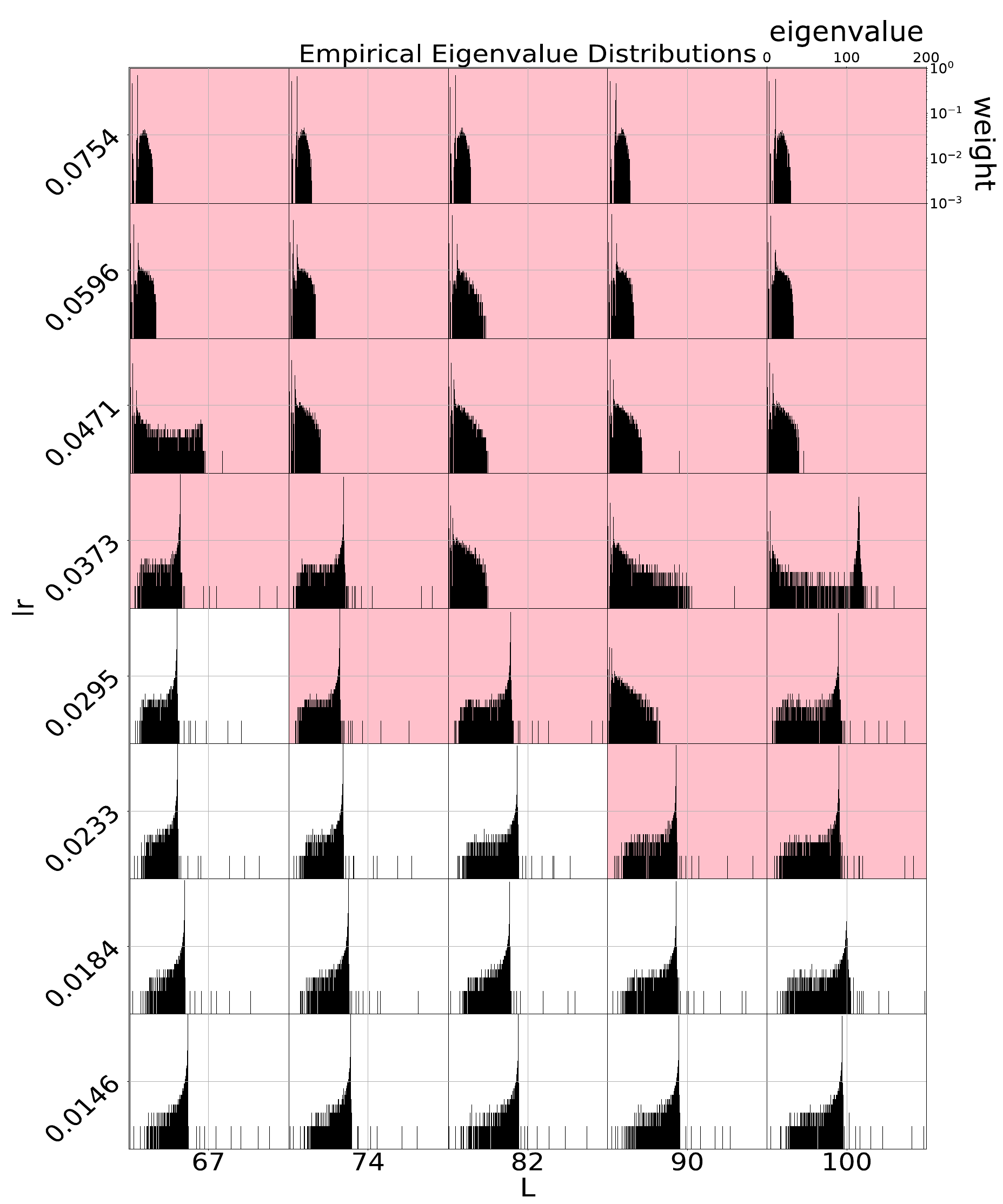}
    \caption{ Histograms of eigenvalue distributions of $H_L$ after training $500$-steps on Fashion-MNIST. All histograms share the x-axis (eigenvalue) and the y-axis (weight). The y-axis is logarithmic. The outer frame's x-axis represents the depth $L$ and its y-axis represents the learning rate $\eta$. Both axes are logarithmic. The histograms in the red region belong to $\eta > 2/L$, and the others belong to $\eta \leq 2/L$.}
    \label{fig:esd_dFIM}
\end{figure}

\section{A Short Introduction to Asymptotic Freeness for Machine Learning}
\newcommand{\C}{{\mathbb{C}}}

\subsection{Comparison with Classical Probability Theory}
Asymptotic freeness is the vital notion in free probability theory.  
In order to introduce asymptotic freeness to readers in the machine learning community, we explain the freeness by comparing free probability theory to classical probability theory. 
We refer readers to \citep{Mingo2017free} or \citep{voiculescu1992free} for the detail.

\begin{table}[h]
\centering
\begin{tabular}{|l|l|l|l|l|l}
\hline
            & r.v.        &  moments  & for multiple r.v.s  & independence  \\ \hline
classical   &  $X$             &  $\E[X^k]$ &  joint distribution    & decomposition of joint distribution     \\
free        &  $A$             &  $\tr[A^k]$ &  joint moments &   decomposition of joint moments \\
\hline
\end{tabular}
\caption{Comparison of free and classical probability theory.}
\label{table:free-classical}
\end{table}

Firstly, consider a matrix $A \in M_M(\C)$. We denote by $A^*$ the adjoint matrix, that is, complex-conjugate transpose matrix, of $A$.  Assume that $A$ is self-adjoint, that is, $A^*=A$.
Then the spectral distribution, denoted by $\mu_A$,  is given by 
\begin{align}  
    \mu_A = \frac{1}{N}\sum_{\lambda \in \sigma(A) } \delta_{\lambda},
\end{align}
where $\sigma(A) = \{ \lambda \in \C \mid  A - \lambda I \text{ \ is invertible }  \} \subset \R$, and $I$ is the identity operator.
We emphasize that the spectral distribution is determined by its moments. That is,  a distribution $\mu$ is equal to $\mu_A$ if and only if 
\begin{align}\label{align:moments-mat}
     \tr(A^k)  =   \int x^k \mu(dx)  \ (k \in \N),
\end{align}
where $\tr$ is the normalized trace so that $\tr(I) = 1$.
In other words, the family of moments  $\tr(A^k)$  $(k \in \N)$ has the same information as the spectral distribution $\mu_A$.
Now consider  a counterpart of the spectral distribution at classical probability theory.
Let $X$ be a real random variable and $\mu_X$ be the distribution of $X$.
If $\mu_X$ is compactly supported, the distribution $\mu_X$ is determined by the moments:
\begin{align}\label{align:moments-rv}
    \E[X^k] = \int x^k \mu_X(dx) \ (k \in \N).
\end{align}
By comparing \eqref{align:moments-mat} and \eqref{align:moments-rv}, we see that the self-adjoint operator $A$ corresponds to a real random variable and the spectral distribution $\mu_A$ corresponds to the distribution of the random variable.

Secondly, consider multiple matrices.
Let $A_1, A_2 \in M_M(\C)$ be non-commuting self-adjoint matrices.  As an example, consider the distribution of the sum of them. The spectral distribution of the sum $A_1 + A_2$ is determined by its moments
\begin{align}\label{align:moment-sum-mat}
\tr[ (A_1+A_2)^k ] =  \sum_{i_1, i_2 ,\dots, i_k \in \{1,2\} } \tr[ A_{i_1} A_{i_2} \cdots A_{i_k} ]   \ (k \in \N).
\end{align}
Hence the trace of all words on $A_1$ and $A_2$ determines the spectral distribution $\mu_{A_1 + A_2}$.
Now consider its counterpart at classical probability theory. Let $X_1$ and $X_2$ be real random variables.  Then 
\begin{align}\label{align:moment-sum-rv}
\E[ (X_1+X_2)^k ] =  \sum_{m=1}^k \begin{pmatrix} k \\ m\end{pmatrix} \tr[ X_1^m X_2^{k-m} ]   \ (k \in \N).
\end{align}
By comparing \eqref{align:moment-sum-mat} and \eqref{align:moment-sum-rv}, we need much more information to determine the distribution for non-commuting matrices  than for commuting real random variables. 
Therefore, we extend the definition of the joint distribution in a different way from that in the classical probability theory.
For a family of matrices $(A_j)_{j \in J}$, its \emph{joint moments} are trace of all words in the family, which are given by
\begin{align}
    \tr[ A_{j_1} A_{j_2} \cdots A_{j_k}]    \  (j_1, j_2, \dots, j_k \in J, \ k \in \N).
\end{align}
Note that we use the joint moments as a counterpart of the joint distribution, which is a probability distribution in classical probability theory and does not exist for non-commuting matrices.

Lastly, consider the notion of independence.
The independence in the classical probability theory means that the joint distribution of multiple random variables decomposed to the product of marginal distributions of each random variable.
Since we consider the joint moments for multiple operators,  we extend the concept of independence as a decomposition law of joint moments to each operator's moments.
The freeness is one of the decomposition laws of joint moments (see \cref{table:free-classical}).

\subsection{Freeness}
Summarizing above, we formulate the free algebra, the tracial state,  and introduce the freeness.

\begin{defn}
A \emph{unital $*$-algebra over $\C$} is a nonempty set  $\mc{A}$ equipped with a triplet $( +, \cdot, *)$ satisfying  the following conditions.
\begin{enumerate}
    \item  $(\mc{A}, +, *)$ is an associative but possibly noncommutative algebra over $\C$, where $+$ is the addition, $\cdot$ is the multiplication, $0 \in \C$ is equal to the additive identity, and $1 \in \C$ is equal to the multiplicative identity. 
    We omit the symbol $\cdot$ and simply write $x \cdot y$ as $xy$ for  $x,y \in \mc{A}$.
    \item  The map $* \colon \mc{A} \to \mc{A}$ satisfies the following:
    \begin{align}
        (xy)^* &= y^*x^*, \\
        (\alpha x + \beta y)^* &= \overline{\alpha} x^* + \overline{\beta}y^*,
    \end{align}
    for any $x,y \in \mc{A}$, $\alpha, \beta \in \C$, where $\overline{\alpha}$ is the complex conjugate of $\alpha$. 
\end{enumerate}
\end{defn}

\begin{example}
The matrix algebra $M_n(\C)$ of $n \times n$ matrices over $\C$ is a unital $*$-algebra for any $n \in \N$, where $A*$ is the adjoint matrix of  $A  \in M_n(\C)$, which is the $n \times n$ matrix obtained from $A$ by taking the transpose and the complex conjugate of the entries.
The matrix algebra $M_n(\R)$ of $n \times n$ matrices over $\C$ is a subalgebra of $M_n(\C)$ as an algebra over $\R$.  Now for $A \in M_n(\R)$,  $A^* = A^\top$.
\end{example}

\begin{example}
Let us denote by  $\C \langle Z_j \mid j \in J \rangle$   the free $\C$-algebra of indeterminates $(Z_j)_{j \in J}$, which is the algebra of polynomials of noncommutative variables given by the weighted sum of words as follows: 
\begin{align}
    \C \langle Z_j \mid j \in J \rangle  = \{ \alpha_\emptyset 1 +  &\sum_{k =1}^\infty \sum_{j_1, \dots, j_k \in J } \alpha_{j_1, j_2, \dots, j_k} Z_{j_1} Z_{j_2}\dots Z_{j_k},\\  &\text{\  where $\alpha_* \in \C$ is zero except for finite number of indices } \} .
\end{align}
We introduce an adjoint operation $*$ on $\C \langle Z_j \mid j \in J \rangle$ using universality of the free product (see \citep{voiculescu1992free}) by   
\begin{align}
    (\alpha_{j_1, j_2, \dots, j_k} Z_{j_1} Z_{j_2}\cdots Z_{j_k} )^* = \overline{\alpha_{j_1, j_2, \dots, j_k}}  Z_{j_k} \cdots Z_{j_2}Z_{j_1}.
\end{align}
\end{example}

Next, we introduce the tracial state, which is an abstracted notion of the normalized trace of matrices and the expectation operator to random variables.
\begin{defn}
A  \emph{tracial state} $\tau$ on a unital $*$-algebra $\mc{A}$ is a $\C$-valued map  satisfying the following conditions.
\begin{enumerate}
    \item  $\tau(1) = 1$.
    \item $\tau(\alpha a + \beta b) = \alpha \tau(a) + \beta \tau(b)$  $(a,b \in \mc{A}, \alpha,\beta \in \C)$.
    \item $\tau(a^*) = \overline{\tau(a)}$   $(a \in \mc{A})$.
    \item $\tau(a^*a) \geq 0$  $(a \in \mc{A})$
    \item $\tau(ab) = \tau(ba)$  $(a,b \in \mc{A})$.
\end{enumerate}
\end{defn}
The first condition is the normalization so that the total volume becomes $1$.
The last one is called the tracial condition.

\begin{defn}
A \emph{noncommutative probability space} (NCPS, for short) is a pair of a unital $*$-algebra $\mc{A}$ and a tracial state $\tau$ on $\mc{A}$.
\end{defn}

Here we have prepared to introduce the freeness. For readers' convenience, we introduce a simpler version of the freeness that is the minimum necessary to read the primary material.
\begin{defn}\label{defn:freeness} (Freeness)
Given two families $a = (a_j)_{j \in I}$  and $b = (b_j)_{j \in J}$ of elements in $\mc{A}$ are said to be $\emph{free}$ (or \emph{free independent} ) with respect to $\tau$ if the following decomposition of joint moments  follows:
For any $k \in \N$, any $p_1, p_2, \dots, p_k \in \C \langle X_i \mid i \in I \rangle $, and any $q_1, q_2, \dots, q_k \in \C \langle Y_j \mid j \in J \rangle$, it holds that
\begin{align}
    \tau[ p_1(a)q_1(b)p_2(a) q_2(b)  \cdots p_k(a) q_k(b)] = 0
\end{align}
if $\tau[p_m(a)]=\tau[q_m(b)] = 0$ ($m=1, 2, \dots, k$).
\end{defn}
See \citep{voiculescu1992free} for the full definition of the freeness for more general cases.

\begin{example}
Assume that  $a$ and $b$ are free elements in a NCPS ($\mc{A}, \tau$), that is, we consider the case $|I|=|J|=1$ in \cref{defn:freeness}. Write $x^\circ = x - \tau(x)$ for $x \in \mc{A}$.
Then 
\begin{align}
\mathrm{Cov}(a,b):= \tau(ab) - \tau(a)\tau(b) = \tau(a^\circ b^\circ)  + \tau(a^\circ)\tau(b) + \tau(a) \tau(b^\circ) = 0 .
\end{align}
Here we use the freeness to eliminate the term $\tau(a^\circ b^\circ)$.
From this equation, we see that free variables are uncorrelated.
The difference between freeness and classical independence appears in the decomposition of higher moments such as $\tau(abab)$:
\begin{align}
    \tau(abab) = \V[a]\E[b]^2 + \E[a]^2\V[b] + \E[a]^2\E[b]^2,
\end{align}
where $\E[a]=\tau(a)$ and $\V[a]=\tau(a^2) - \tau(a)^2$. The decomposition rule is different from that in the classical probability given by  $\E[XYXY] = \E[X^2Y^2] = \E[X^2]\E[Y^2]$ for independent random variables $X$ and $Y$.
\end{example}

\subsection{Infinite Dimensional Approximation of Random Matrices}

Here we introduce the relation between freeness and random matrices.

\begin{defn}
Let $ A_i(M), B_j(M) \in M_M(\C)$ ($M \in \N, i \in I, j \in J$).
Then the families $(A_i)_{i \in I}$ and $(B_j)_{j \in J}$ are said to be \emph{asymptotically free} as $M \to \infty$ if there exists $\mc{A} ,\tau$, and $a_i \in \mc{A} \ (i \in I)$ and $b_j \in \mc{A} \ (j \in J)$  so that 
\begin{align}
    &\lim_{M \to \infty}\tr(A_i(M)^k) = \tau(a_i^k)  \ (k \in \N, i \in I),\\
    &\lim_{M \to \infty}\tr(B_j(M)^k) = \tau(b_j^k)  \ (k \in \N, j \in J),\\
    &\text{and, $(a_i)_{ i \in I}$ and $(b_j)_{j \in J}$ are free.} 
\end{align}
\end{defn}

Here we introduce a known result in free probability theory.
\begin{prop}[{\cite[Prop.\,3.5]{Hiai2000asymptotic}} ]\label{prop:asymptotic-free}
For each $M \in \N$, consider the following matrices. 
Let $U(M)$ be random matrix uniformly distributed on  $M \times M$ unitary matrices (resp.\,orthogonal matrices).
Let $A(M)$ and $B(M)$ be complex (resp.\,real) self-adjoint random matrices independent of $U(M)$.
Assume that there exist two compactly supported distributions $\mu$ and $\nu$ such that the following limits hold almost surely.
\begin{align}
    \lim_{M\to\infty}\tr[A(M)^k]&=  \int x^k \mu(dx) \  (k \in \N), \\
    \lim_{M\to\infty}\tr[B(M)^k]&=  \int x^k \nu(dx) \  ( k \in \N).
\end{align}
Under the above conditions, it holds that $B(M)$ and $U(M)^* A(M) U(M)$ are asymptotically free as $M \to \infty$ almost surely. 
Furthermore, when $A(M)$ and $B(M)$ are positive definite, then the limit distribution of $B(M)^{1/2}U(M)^*A(M)U(M)B(M)^{1/2}$ is the multiplicative free convolution $\mu \boxtimes \nu$.
\end{prop}

Note that random matrices $A(M)$ and $B(M)$ do not have to be independent in \cref{prop:asymptotic-free}.

\subsection{Application to the FIM}
Recall that the propagation of the conditional FIM is given by the following equation:
\begin{align}
\Hmltn_{\ell +1 } =   \hat{q}_{\ell}I + W_{\ell+1} D_{\ell}\Hmltn_{\ell}D_{\ell}W_{\ell+1}^\top.
\end{align}
Let $A_\ell := W_\ell^* H_\ell W_\ell  = \hat{q}_{\ell-1}I + D_{\ell-1}\Hmltn_{\ell-1}D_{\ell-1}$. Note that $W_\ell^*=W_\ell^\top$. Then, firstly, $W_{\ell}$ and $A_\ell$ are independent.
Secondly, recall that $D_\ell$ and $(W_\ell, W_\ell^*)$ are assumed to be asymptotic free (see Assumption~3.1). Thirdly, each $W_\ell$ is uniformly distributed on orthogonal matrices. 
By the above conditions,  it holds that $W_\ell A_\ell W_\ell^*$ and $D_\ell^2$ are asymptotic free as the limit $M \to \infty$ by \cref{prop:asymptotic-free}.
Then the limit spectral distribution of $D_\ell H_\ell D_\ell = D_\ell W_\ell A_\ell W_\ell^* D_\ell$  is equal to $\mu_\ell \boxtimes  \nu_\ell$, where $\mu_\ell$ (resp\,$\nu_\ell$) is the limit spectral distribution of $H_\ell$ (resp.\,$D_\ell^2$).
Then we get the following desired recursive equation:
\begin{align}
    \mu_{\ell+1} = \affine{\sigma_{\ell+1}^2}{q_\ell}(\nu_\ell \boxtimes \mu_\ell).
\end{align}
\section{Analysis}
Firstly, let us review on the following proposition known in free probability theory.
\begin{prop}[\cite{Belinschi2003atoms}]\label{prop:atom}
Let  $\mu$ and $\nu$ be compactly supported probability distributions on $\R$.  Then $\mu \boxtimes \nu$ has an atom at $c \in \R$ if and only if the following three conditions hold : (i) $a \in \R$ (resp.\,$b \in \R$) is an atom of $\mu$ (resp.\,$\nu$), (ii) $c = ab$, and (iii)  $\mu(\{a\}) +  \nu(\{b\}) -1 > 0$.
Furthermore, if $c$ is an atom then $\mu \boxtimes \nu(\{c\}) = \mu(\{a\}) +  \nu(\{b\}) -1$.
\end{prop}

Then we have  the following recurrence equation of the maximum eigenvalue.
\begin{lemma}\label{lemma:atom}
Fix $L \in \N$.
Let  $\beta_1 = 1 $ and $\beta_{\ell}  =   1  - \sum_{k=1}^{\ell-1}(1 - \alpha_k) $  for $\ell \geq2$.
Assume that  $ \beta_L > 0$.
Then for any $\ell \leq  L$,  the value $|| \mu_\ell||_\infty$  is an atom of $\mu_\ell$ with weight  $\beta_{\ell}$.
Furthermore, we have $ || \mu_\ell ||_\infty  =    q_{\ell-1} + \sigma_\ell^2 \gamma_{\ell-1}  || \mu_{\ell-1} ||_\infty$ for $\ell \neq 1$.
\end{lemma}
\begin{proof}
Let us define $\lambda_\ell \in \R$ recursively by $\lambda_{\ell} = q_{\ell-1} + \sigma_\ell^2 \gamma_{\ell-1} \lambda_{\ell-1} (\ell \geq 2)$ and $\lambda_1=q_0$. 
Firstly we prove that $\lambda_{\ell}$ is an atom with weight $\beta_\ell$ of $\mu_\ell$ for $\ell \leq L$.
In the case $\ell = 1$, we have $\mu_1 = \delta_1 = \delta_{\lambda_1}$. 
Fix $\ell > 1$ and assume that $\lambda_{\ell-1}$ is an atom of $\mu_{\ell-1}$ with weight $\beta_{\ell-1}$. 
Now $\beta_{\ell-1} + \alpha_{\ell-1} - 1 = \beta_\ell \geq \beta_L > 0$.
Hence by \cref{prop:atom},  $\nu_{\ell-1} \boxtimes \mu_{\ell-1}$ has an atom $\gamma_{\ell-1}\lambda_{\ell-1}$ with weight $\beta_\ell$. Therefore $\mu_{\ell}$ has the atom $\lambda_\ell$ with weight $\beta_\ell$. 
The claim follows from the induction on $\ell$.
To complete the proof, we only need to show that $\lambda_\ell = ||\mu_\ell||_\infty$. Clearly $\lambda_\ell \leq  ||\mu_\ell||_\infty$.
Note that $||\mu \boxtimes \nu ||_\infty \leq ||\mu||_\infty || \nu ||_\infty$.
Then $||\mu_\ell||_\infty \leq  q_{\ell-1}   + \sigma_\ell^2 \gamma_{\ell-1}||\mu_{\ell-1}||_\infty$.
Thus it holds that $ ||\mu_\ell||_\infty  \leq \lambda_\ell$ since $||\mu_1||_\infty = q_0 =  \lambda_1$.
Hence the claim follows.  
\end{proof}

Now we have prepared to prove the desired theorem. 
\begin{thm}\label{thm:maximum-appendix}
%Consider \cref{assumption:simplify}  and  \ref{assumption:eps}.
Consider Assumption 4.1 and 4.2.
Then for sufficiently larger $L$,  it holds that $||\mu_L||_\infty$  is an atom of $\mu_L$ with weight $1- (L-1)(1-\alpha_{L-1})$,  and 
\begin{align}
     \lim_{L\to\infty}L^{-1}||\mu_L||_\infty  = q \eps_2^{-1} \left[ 1 - \exp\left(-\eps_2\right) \right].
\end{align}
In particular, the limit has the expansion $q(1 - \eps_2/2) + O(\eps_2^2)$ as the further limit $\eps_2 \to 0$.
\end{thm}
\begin{proof}
Since $\eps_1 < 1$, we have $ 1- \alpha_{L-1} < (L-1)^{-1}$ for sufficiently large $L$. Then $\beta_L = 1 -(L-1)(1-\alpha_{L-1}) >0$. Hence by \cref{lemma:atom},  for any $\ell \leq L$,  it holds that $||\mu_\ell||_\infty$ is an atom of $\mu_\ell$ and  $||\mu_L||_\infty = q_L  \sum_{\ell=0}^{L-1}(\sigma_L^2\gamma_{L-1})^\ell$. Then by the same discussion as  Proposition~4.4, the assertion follows.
\end{proof}

\cref{thm:maximum-appendix} shows that the maximum eigenvalue of the conditional FIM $H_L$ is $O(L)$ as $L \to \infty$.
Furthermore, we emphasize that the weight $1 - (L-1)(1-\alpha_{L-1}) \sim 1 - \eps_1$ of the maximal eigenvalue $||\mu_{L}||$ is close to $1$.
Therefore, eigenvalues of the dual conditional FIM $H_L$  concentrates around $qL(1+\eps_2/2)$, and the dual FIM approximates the scaled identity operator.
Clearly the same property holds for  non-zero eigenvalues of the conditional FIM $\mc{I}(\theta|x)$.

\end{document}